\documentclass{article}

\usepackage{microtype}
\usepackage{graphicx}
\usepackage{subfigure}
\usepackage{booktabs} % for professional tables
\usepackage{comment}
\usepackage{amsfonts}
\usepackage{cite}
\usepackage[hyphens]{url}  % DO NOT CHANGE THIS
\usepackage{graphicx} % DO NOT CHANGE THIS
\usepackage{amsmath,amsthm,amsfonts,amssymb}
\usepackage{enumerate}
\usepackage{mathabx}
\usepackage{multirow}
\usepackage{epstopdf}
\usepackage{enumitem}
\usepackage{hyperref}

\DeclareMathOperator*{\argmin}{arg\,min}

\newtheorem{theo}{Theorem}
\newtheorem{definition}{Definition}
\newtheorem{lemma}{Lemma}

\usepackage[accepted]{icml2020}

\icmltitlerunning{Learning with Multiple Complementary Labels}

\begin{document}

\twocolumn[
\icmltitle{Learning with Multiple Complementary Labels}



\icmlsetsymbol{equal}{*}

\begin{icmlauthorlist}
\icmlauthor{Lei Feng}{equal,first}
\icmlauthor{Takuo Kaneko}{equal,second,third}
\icmlauthor{Bo Han}{fourth,third}
\icmlauthor{Gang Niu}{third}
\icmlauthor{Bo An}{first}
\icmlauthor{Masashi Sugiyama}{third,second}
\end{icmlauthorlist}

\icmlaffiliation{first}{School of Computer Science and Engineering, Nanyang Technological University, Singapore}
\icmlaffiliation{second}{The University of Tokyo}
\icmlaffiliation{third}{RIKEN Center for Advanced Intelligence Project}
\icmlaffiliation{fourth}{Department of Computer Science, Hong Kong Baptist University}

\icmlcorrespondingauthor{Lei Feng}{feng0093@e.ntu.edu.sg}
%\icmlcorrespondingauthor{Eee Pppp}{ep@eden.co.uk}

\icmlkeywords{Machine Learning, ICML}

\vskip 0.3in
]



\printAffiliationsAndNotice{\icmlEqualContribution $\displaystyle^{{\star}}$Work done when LF was an intern at RIKEN AIP and TK belonged to The University of Tokyo and RIKEN AIP.} % otherwise use the standard text.

\begin{abstract}
A \emph{complementary label}~(CL) simply indicates an incorrect class of an example, but learning with CLs results in multi-class classifiers that can predict the correct class.
Unfortunately, the problem setting only allows a \emph{single} CL for each example, which notably limits its potential since our labelers may easily identify \emph{multiple CLs}~(MCLs) to one example.
In this paper, we propose a novel problem setting to allow MCLs for each example and two ways for learning with MCLs. % This paper for the first time studies the important problem setting that allows
In the first way, we design two \emph{wrappers} that \emph{decompose} MCLs into many single CLs, so that we could use any method for learning with CLs.
However, the supervision information that MCLs hold is conceptually diluted after decomposition. Thus, in the second way, we derive an \emph{unbiased risk estimator}; minimizing it processes each set of MCLs as a whole and possesses an \emph{estimation error bound}.
We further improve the second way into minimizing properly chosen upper bounds.
Experiments show that the former way works well for learning with MCLs but the latter is even better.
%Complementary-label learning solves the problem where each training example is supplied with a complementary label, which only specifies one of the classes that the example does \textit{not} belong to. Although seminal works have demonstrated that an unbiased estimator of the original classification risk can be obtained from only complementarily labeled data, they are all restricted to the case where each example is associated with exactly one complementary label. It would be more promising to learn from multiple complementary labels simultaneously, as the supervision information would be richer if more complementary labels are provided. So far, it is still unknown whether there exists an unbiased risk estimator for learning with multiple complementary labels simultaneously. In this paper, we will give an affirmative answer by deriving the first unbiased risk estimator for learning with multiple complementary labels. We further theoretically derive an estimation error bound of the proposed approach. %, and show that the optimal parametric convergence rate is achieved. Finally, we experimentally demonstrate the effectiveness of the proposed approach.
\end{abstract}

\section{Introduction}
Ordinary machine learning tasks generally require massive data with accurate supervision information, while it is expensive and time-consuming to collect the data with high-quality labels. To alleviate this problem, the researchers have studied various weakly supervised learning frameworks~\cite{zhou2018brief}, including \emph{semi-supervised learning}~\cite{Chapelle2006semi,zhu2009introduction,niu2013squared,miyato2018virtual,li2019safe}, \emph{positive-unlabeled learning}~\cite{elkan2008learning,Plessis2014Analysis,Plessis2015Convex,Kiryo2017Positive,sakai2017semi,sakai2018semi}, \emph{noisy-label learning}~\cite{menon2015learning,han2018masking,han2018co,xia2019anchor,Wei_2020_CVPR}, \emph{partial label learning}~\cite{cour2011learning,zhang2015solving,feng2018leveraging,feng2019partial,feng2019partial2}, \emph{positive-confidence learning}~\cite{Ishida2018Binary}, \emph{similar-unlabeled learning}~\cite{Bao2018Classification}, and \emph{unlabeled-unlabeled classification}~\cite{lu2019on,lu2020mitigating}.

Here, we consider another weakly supervised classification framework called \emph{complementary-label learning}~\cite{ishida2017learning,yu2018learning,Ishida2019Complementary}. In complementary-label learning, each training example is supplied with a \emph{complementary label} (CL), which specifies one of the classes that the example does \emph{not} belong to. Compared with ordinary labels, it is obviously easier to collect CLs. Recently, complementary-label learning has been applied to online learning~\cite{kaneko2019online} %generative discriminative learning~\cite{xu2019generative},
and medical image segmentation~\cite{rezaei2019recurrent}. In addition, another potential application of learning with CLs would be data privacy. For example, collecting some survey data may require extremely private questions~\cite{ishida2017learning,Ishida2019Complementary}. It may be difficult for us to directly obtain the true answer (label) to the question. Nonetheless, it would be mentally less demanding if we ask the respondent to provide some incorrect answers. Besides, the respondent may provide multiple incorrect answers, rather than exactly one. In this case, \emph{multiple complementary labels} (MCLs) would be more widespread than a single CL. %Thus, this paper proposes to generalize the previous problem setting of complementary-label learning to allow MCLs for each example. %the previous problem setting to allow MCLs for each example.

%As stated in~\cite{Ishida2019Complementary}, complementary-label learning could be applied to deal with the problems from crowdsourcing and data privacy.
%To the best of our knowledge, there exist only three published works
%There are some seminal works on learning with complementary labels. The first work was given by~\citet{ishida2017learning}. They assumed that the complementary labels are chosen in an even way and showed that an unbiased risk estimator can be obtained from only complementarily labeled data. They further provided theoretical analysis with a statistical consistency guarantee. Later, a different assumption on the complementary labels was adopted by~\citet{yu2018learning}. They assumed that the complementary labels are chosen in an uneven way, and proposed a different formulation that corrects the loss by estimating a class transition probability matrix. More recently, a novel complementary-label learning approach~\cite{Ishida2019Complementary} was proposed, which not only provides an unbiased risk estimator, but also can be used for arbitrary losses and models. Although these seminal works have provided solid theoretical foundations and achieved promising performance for complementary-label learning, they are all restricted to the case where each example is associated with exactly one complementary label. In other words, they do not take into account learning with multiple complementary labels simultaneously.
%
In this paper, we propose a novel problem setting (Section~\ref{data_generation}) that allows MCLs for each example, and provide a real-world motivation (Section~\ref{real_world}). Although existing complementary-label learning approaches \cite{ishida2017learning,yu2018learning,Ishida2019Complementary} have provided solid theoretical foundations and achieved promising performance, they are all restricted to the case where each example is associated with a single CL. To learn with MCLs, we first design two wrappers (Section~\ref{wrapper}) that decompose each example with MCLs into multiple examples, each with a single CL, in different manners. With the two wrappers, we are able to use arbitrary ordinary complementary-label learning approaches for learning with MCLs. However, the derived data with many single CLs may not match the assumed data distribution for complementary-label learning \cite{ishida2017learning,Ishida2019Complementary}. In addition, the supervision information would be conceptually diluted after decomposition. %in the process of the SGD-like algorithm.
%since ordinary complementary-label learning approaches can only learn from a single CL for each example at a time, wrappers they may lose some potentially useful information from other complementary labels. In other words, the supervision information would be conceptually diluted after decomposition.

In order to solve the above problems, we further propose an unbiased risk estimator (Section~\ref{risk_estimator}) for learning with MCLs, which processes each set of MCLs as a whole. Our risk estimator is conceptually consistent, and builds a prototype baseline for the new problem setting that may inspire more specially designed methods for this new setting in the future. %In this way, we can not only match the data distribution with MCLs, but also avoid the dilution of the supervision infortion. %Specifically, we first make an assumption on the generation process of the data with MCLs, and show that this assumption is practically reasonable. Based on this assumption, we prove that an unbiased risk estimator can be obtained from the data with MCLs.
Then, we theoretically derive an estimation error bound, which guarantees that the empirical risk minimizer converges to the true risk minimizer with high probability as the number of training data approaches infinity.
Furthermore, we improve the risk estimator into minimizing properly chosen upper bounds for practical implementation (Section \ref{implement}), and we show that they bring benefits to gradient update. Experimental results show that the wrappers work well for learning with MCLs while the (improved) risk estimator is even better on various benchmark datasets.
%Therefore, it is of significant importance to derive an unbiased risk estimator for learning with multiple complementary labels simultaneously, as the supervision information will be richer if more complementary labels are provided.
%In this paper, we propose a novel approach with the first unbiased risk estimator for learning with MCLs simultaneously. Specifically, we first make an assumption on the generation process of examples with MCLs, and show that this assumption is practically reasonable. Based on the assumption, we prove that an unbiased risk estimator can be obtained from the data with MCLs. Furthermore, we theoretically derive an estimation error bound. %, and show that our proposed approach achieves the optimal parametric convergence rate. Finally, we demonstrate the practical effectiveness of the proposed approach through experiments.
%and do not provide any theoretical foundations for learning with MCLs.
%and cannot generalize to MCLs.
%extend this problem to a biased setting where they assume that complementary labels are selected with bias,
%The first work is from. They have shown that an unbiased estimator of ordinary classification risk can be recovered from only complementarily labeled data.

\section{Related Work}
In this section, we introduce some notations and briefly review the formulations of multi-class classification and complementary-label learning. %and partial label learning. Meanwhile, we discuss the relations among them.
\subsection{Multi-Class Classification}
Suppose the feature space is $\mathcal{X}\in\mathbb{R}^d$ with $d$ dimensions and the label space is $\mathcal{Y}=\{1,2,\dots,k\}$ with $k$ classes,
the instance $\boldsymbol{x}\in\mathcal{X}$ with its class label $y\in\mathcal{Y}$ is sampled from an unknown probability distribution with density $p(\boldsymbol{x},y)$. %and the given dataset $\mathcal{D}$ is represented by $\{(\boldsymbol{x}_i,y_i)\}_{i=1}^n$.
Ordinary multi-class classification aims to induce a learning function $f(\boldsymbol{x}):\mathbb{R}^d\rightarrow\mathbb{R}^k$ that minimizes the classification risk:
\begin{align}
\label{risk}
R(f) = \mathbb{E}_{p(\boldsymbol{x},y)}\big[\mathcal{L}\big(f(\boldsymbol{x}),y\big)\big],
\end{align}
where $\mathcal{L}\big(f(\boldsymbol{x}),y\big)$ is a multi-class loss function. The predicted label is given as $\hat{y}=\text{argmax}_{y\in\mathcal{Y}}f_y(\boldsymbol{x})$, where $f_y(\boldsymbol{x})$ is the $y$-th coordinate of $f(\boldsymbol{x})$.
\subsection{Complementary-Label Learning}
Suppose the dataset for complementary-label learning is denoted by $\{(\boldsymbol{x}_i,\widebar{y}_i)\}_{i=1}^n$, where $\widebar{y}_i\in\mathcal{Y}$ is a complementary label of $\boldsymbol{x}_i$, and each complementarily labeled example is sampled from $\widebar{p}(\boldsymbol{x},\widebar{y})$. %Having a direct unbiased risk estimator is very important. It not only guides the learning process, but also is helpful to the validation process. Since we do not have ordinary labels when only complementarily labeled data are provided, we cannot follow the general validation method that utilizes the 0-1 error or classification accuracy. In this case, an unbiased estimator of the original classification risk in Eq.~(\ref{risk}) allows us to conduct cross validation based on the empirical risk.
%In order to provide an unbiased estimator of the classification risk from only complementarily labeled data, it is important to make some assumptions on $\widebar{p}(\boldsymbol{x},\widebar{y})$.
\citet{ishida2017learning,Ishida2019Complementary} assumed that $\widebar{p}(\boldsymbol{x},\widebar{y})$ is expressed as:
%The assumption used by~\citet{ishida2017learning,Ishida2019Complementary} is expressed as
\begin{gather}
\label{assumption}
\textstyle
\widebar{p}(\boldsymbol{x},\widebar{y})=\frac{1}{k-1}\sum\nolimits_{y\neq\widebar{y}}p(\boldsymbol{x},y).
\end{gather}
This assumption implies that all other labels except the correct label are chosen to be the complementary label with uniform probabilities. This is reasonable as we do not have extra labeling information expect a complementary label. Under this assumption, it was proved by~\citet{ishida2017learning} that an unbiased estimator of the original classification risk can be obtained from only complementarily labeled data, when the loss function satisfies certain conditions. Specifically, they used the multi-class loss functions with the one-versus-all strategy and the pairwise comparison strategy~\cite{zhang2004statistical}:
\begin{align}
\nonumber
\textstyle \widebar{\mathcal{L}}_{\text{OVA}}\big(f(\boldsymbol{x}),\widebar{y}\big)&= \textstyle \frac{1}{k-1}\sum\nolimits_{y^\prime\neq\widebar{y}}\ell\big(f_{y^\prime}(\boldsymbol{x})\big) + \ell\big(-f_{\widebar{y}}(\boldsymbol{x})\big),\\
\nonumber
\textstyle
\widebar{\mathcal{L}}_{\text{PC}}\big(f(\boldsymbol{x}),\widebar{y}\big)&=
\textstyle \sum\nolimits_{y^\prime\neq\widebar{y}}\ell\big(f_{y^\prime}(\boldsymbol{x})-f_{\widebar{y}}(\boldsymbol{x})\big),
\end{align}
where $\ell(z)$ is a binary loss function that satisfies $\ell(z)+\ell(-z)=1$, such as sigmoid loss $\ell_{\text{S}}(z)=\frac{1}{1+e^z}$ and ramp loss $\ell_{\text{R}}(z)=\frac{1}{2}\max(0,\min(2,1-z))$.

Later, another different assumption was used by~\citet{yu2018learning}. They assumed that all  other  labels  except the correct label are chosen to be the complementary label  with  different  probabilities, and proposed to estimate the class transition probability matrix for model training. %during model training.
Although they showed that the minimizer of their learning objective coincides with the minimizer of the original classification risk, they did not provide an unbiased risk estimator.

Recently, a more general unbiased risk estimator~\cite{Ishida2019Complementary} was proposed, which does not rely on specific losses or models. Their formulation is as follows:
\begin{gather}
\label{free}
\textstyle
\widebar{\mathcal{L}}_{\text{FREE}}\big(f(\boldsymbol{x}),\widebar{y}\big)=\sum\limits_{y=1}^k\mathcal{L}\big(f(\boldsymbol{x}),y\big) - (k-1)\mathcal{L}\big(f(\boldsymbol{x}),\widebar{y}\big).
\end{gather}
For this formulation, they showed that due to the negative term, the empirical risk could be unbounded below, which leads to over-fitting. In order to alleviate this issue, the authors further proposed modified versions by using the max operator and the gradient ascent strategy. %They showed that their risk estimator can be used not only for a learning objective, but also as a validation criterion for the above two approaches.

In summary, although the above methods have provided solid theoretical foundations and achieved promising performance for complementary-label learning, they are all restricted to the case where each example is associated with a single CL. In this paper, we propose a novel problem setting that allows MCLs for each example.
\section{Multiple Complementary Labels}
In this section, we first introduce our problem setting where each example is associated with MCLs, and then provide a corresponding real-world motivation.
\subsection{Data Generation Process}\label{data_generation}
Suppose the given dataset for learning with MCLs is represented as $\widebar{\mathcal{D}}=\{(\boldsymbol{x}_i,\widebar{Y}_i)\}_{i=1}^n$, where $\widebar{Y}_i$ is a set of complementary labels for the instance $\boldsymbol{x}_i$. It is obvious that learning with MCLs is a generalization of complementary-label learning that learns with a single CL. Specifically, if $\widebar{Y}_i$ contains only one complementary label with probability $1$, we obtain a complementary-label learning problem. In addition, if $\widebar{Y}_i$ contains $k-1$ complementary labels where $k$ denotes the total number of classes, we obtain an ordinary multi-class classification problem.
It is easy to know that for all $i$, $\widebar{Y}_i$ cannot be the empty set nor the full label set, hence $\widebar{Y}_i\in\widebar{\mathcal{Y}}$ where $\widebar{\mathcal{Y}} = \{2^{\mathcal{Y}}-\emptyset-\mathcal{Y}\}$ and $|\widebar{\mathcal{Y}}|=2^k-2$.

For the generation process of each example with MCLs, we assume that it relies on the size of the set of MCLs. Let us represent the size of the complementary label set by a random variable $s$, and assume $s$ is sampled from a distribution $p(s)$.
%random variable $s$ representing the size of the complementary label set is sampled from a distribution $p(s)$.
%Intuitively, the variables $s$ and $\boldsymbol{x}$ are independent since they do not influence each other. %while $s$ influences $\widebar{Y}$.
%Since the variable $s$ does not influence the generation process of $\boldsymbol{x}$, i.e., $p(\boldsymbol{x}|s) = p(\boldsymbol{x})$.
%For convenience, we define the set of the subset of the label space $\widebar{\mathcal{Y}}$ whose size is $s$ as
%where $2^\mathcal{Y}$ denotes the power set of the label space $\mathcal{Y}$.
In this way, we assume that each training example $(\boldsymbol{x}_i,\widebar{Y}_i)$ is drawn from the following data distribution:
\begin{gather}
\label{distr}
\widebar{p}(\boldsymbol{x}, \widebar{Y}) = \sum\nolimits_{j=1}^{k-1}p(s=j)\widebar{p}(\boldsymbol{x},\widebar{Y} \mid s=j),
\end{gather}
where
\begin{gather}
\nonumber
\widebar{p}(\boldsymbol{x},\widebar{Y} \mid s=j) := \begin{cases}
\frac{1}{\tbinom{k-1}{j}} \sum_{y \notin \widebar{Y}} p(\boldsymbol{x}, y), & \text{ if } |\widebar{Y}|=j,  \\
0, & \text{ otherwise}.
\end{cases}
\end{gather}
It is clear that when $p(s=1)=1$, our introduced distribution reduces to the assumed distribution (e.g., Eq.~(\ref{assumption})) in ordinary complementary-label learning approaches~\cite{ishida2017learning,Ishida2019Complementary}. Then, we show that $\widebar{p}(\boldsymbol{x},\widebar{Y})$ is a valid probability distribution by the following theorem.
\begin{theo}
\label{theo1}
The following equality holds:
\begin{gather}
\int_{\widebar{\mathcal{Y}}}\int_{\mathcal{X}}{\widebar{p}(\boldsymbol{x},\widebar{Y})} \mathrm{d}\boldsymbol{x}\ \mathrm{d}\widebar{Y}=1.
\end{gather}
\end{theo}
The proof is provided in Appendix A.1.
\subsection{Real-World Motivation}\label{real_world}
%Although we have shown that it is a valid probability distribution, a question still remains: Is there a real-world application that satisfies the assumed distribution? We give an affirmative answer by providing a concrete example.
Here, we present a real-world motivation for the assumed data distribution.

Since directly choosing the correct label is hard for labelers, it would be easier if a labeling system can randomly choose a label set and ask labelers whether the correct label is included in the proposed label set or not. %Besides, if the correct label is included in the suggested label set, the system asks a labeler to identify the correct label.
Given a pattern $\boldsymbol{x}$, suppose the labeling system first randomly samples the size $s$ of the proposed label set from $p(s)$, and then randomly and uniformly chooses a specific label set with size $s$ from $\widebar{\mathcal{Y}}$. In this way, the collected label sets that do not include the correct label precisely follow the same distribution as Eq.~(\ref{distr}). We will demonstrate this fact in the following.

We start by considering the case where the labeling system has already sampled the size $s$ of the proposed label set. Then we have the following lemma.
%uniformly and randomly chooses a label set with fixed size $s$, then the gathered label sets that do not include the correct label satisfy the same distribution of Eq.~(\ref{distr}). We will demonstrate this fact in the following. In order to label a pattern $\boldsymbol{x}$, a labeling system proposes a label set $\widebar{Y}$ uniformly and randomly sampled from $\widebar{\mathcal{Y}}_s$. In this case, $p(\widebar{Y}) = \frac{1}{|\widebar{\mathcal{Y}}_s|}$. Since the labeling system should know whether the correct label $y$ of the pattern $\boldsymbol{x}$ is included in the label set $\widebar{Y}$ or not, we model the probability $p(y\in\widebar{Y})$ and give the following lemma.
\begin{lemma}
\label{lemma1}
Given the sampled size $s$ of the proposed label set, for any pattern $\boldsymbol{x}$ with its correct label $y$ and any label set $\widebar{Y}$ with size $s$ (i.e., $|\widebar{Y}|=s$), the following equality holds:
\begin{gather}
p(y\in\widebar{Y} \mid \boldsymbol{x},s)=\frac{s}{k}.
\end{gather}
\end{lemma}
The proof is provided in Appendix A.2.

\begin{theo}
\label{real_world_theo}
In the above setting, the distribution of collected data where the correct label $y$ ($y\in\mathcal{Y}$) is not included in the label set $\widebar{Y}$ ($\widebar{Y}\in\widebar{\mathcal{Y}}$) is the same as Eq.~(\ref{distr}), i.e.,
\begin{gather}
p(\boldsymbol{x},\widebar{Y} \mid y\notin\widebar{Y}) = \widebar{p}(\boldsymbol{x},\widebar{Y}).
\end{gather} %\frac{1}{\tbinom{k-1}{s}}\sum_{y\notin\widebar{Y}}p(\boldsymbol{x},y).
\end{theo}
The proof is provided in Appendix A.3.
\section{Learning with Multiple Complementary Labels}
In this section, we first present two wrappers that enable us to use any ordinary complementary-label learning approach for learning with MCLs. Then, we present an unbiased risk estimator for learning with MCLs as a whole, and establish an estimation error bound.
\subsection{Wrappers}\label{wrapper}
Since ordinary complementary-label learning approaches cannot directly deal with MCLs, it would be natural to ask whether there exist some strategies that can enable us to use any existing complementary-label learning approach for learning with MCLs. %With such strategies, all the complementary-label learning approaches are able to learn with MCLs.

Motivated by this, we propose two wrappers that decompose each example with MCLs into multiple examples, each with a single CL.
%While it would be intuitive to decompose each example with MCLs into multiple examples, each with a single CL.
Specifically, suppose a training example with MCLs is given as $(\boldsymbol{x}_i,\widebar{Y}_i)$ where $\widebar{Y}_i=\{\widebar{y}_1,\widebar{y}_2\}$. Then ordinary complementary label learning approaches may learn from $(\boldsymbol{x}_i,\widebar{y}_1)$ and $(\boldsymbol{x}_i,\widebar{y}_2)$. According to whether decomposition is after shuffling the training set, there are two decomposition strategies (wrappers) when we optimize a loss function by a stochastic optimization algorithm:
%\begin{itemize}
%\setlength{\topsep}{-100pt}
%\setlength{\itemsep}{-5pt}
%\setlength{\parsep}{-5pt}
%\setlength{\parskip}{-5pt}

\textbf{Decomposition after Shuffle.\quad} Given the shuffled training set with MCLs, in each mini-batch, we decompose each example into multiple examples, each with a single CL.

\textbf{Decomposition before Shuffle.\quad} Given the training set with MCLs, we drive a new training set by decomposing each example into multiple examples, each with a single CL. Then, we shuffle the derived training set.
%\end{itemize}

Both the above decomposition strategies enable us to use arbitrary ordinary complementary-label learning approaches for learning with MCLs.
However, the derived training data with many single CLs may not match the originally assumed data distribution (i.e., Eq.~(\ref{assumption})) for complementary-label learning, since these CLs are completely derived from MCLs while the data distribution with MCLs is relevant to the size of each set of MCLs. As a consequence, the learning consistency would no longer be guaranteed even if the complementary-label learning approach inside the wrappers is originally risk-consistent or classifier-consistent.

%not exactly follow Eq.~(\ref{assumption}). %is relevant to the size of each set of MCLs.
\begin{table}[!t]
\centering
\normalsize
	\caption{Supervision information for a set of MCLs (with size $s$).}
	\label{supervision_information}
	\setlength{\tabcolsep}{0.5mm}{
    %\resizebox{.48\textwidth}{!}{
				\begin{tabular}{c|ccc}
					\toprule
					Setting & \#TP & \#FP & \ Supervision Purity\\
					\midrule
					Many single CLs & $s$ & $(k-2)s$ & ${1}/(k-1)$\\
					A set of MCLs & $1$ & $k-s-1$ & ${1}/(k-s)$\\
					\bottomrule
				\end{tabular}
				}
				\vspace{-0.8cm}
\end{table}
Moreover, since ordinary complementary-label learning approaches can only learn with a single CL for each example at a time and treat each example independently, % that the examples are independent, %they may lose some potentially useful information from other complementary labels. In other words,
the supervision information for each set of MCLs would be conceptually diluted. %in the optimization process of SGD-like algorithms.
We demonstrate this issue by Table~\ref{supervision_information}. As shown in Table~\ref{supervision_information}, there are two settings according to whether to decompose a set of MCLs into many single CLs or not. Since all the non-complementary labels have the possibility to be the correct label, we specially count how many times the correct label serves as a non-complementary label (denoted by \#TP), and how many times the other labels except the correct label serve as a non-complementary label (denoted by \#FP). Then the supervision purity is calculated by (\#TP)/(\#TP+\#FP).

Clearly, the wrappers follow the setting where a set of MCLs is decomposed into many single CLs. If the size of the set of MCLs is $s$, then \#TP equals $s$, since the correct label would serve as a non-complementary label for $s$ times after decomposition, and the other labels except the correct label would serve as a non-complementary label for $(k-s-1)s+s(s-1)=(k-2)s$ times, hence the supervision purity would be $s/(s+(k-2)s)=1/(k-1)$. However, for the setting where the set of MCLs is not decomposed, we can easily know that the correct label serves as a non-complementary label once, and the other labels expect the correct label serve as a non-complementary label for $k-s-1$ times, hence the supervision purity is $1/(k-s)$. These observations clearly show that the supervision information is diluted after decomposing MCLs ($s\geq 2$), which also motivate us to take a set of MCLs as a whole set. In the following, we will introduce our proposed unbiased risk estimator, which is able to learn with MCLs as a whole.
%For example, if an example $\boldsymbol{x}_i$ is provided with two complementary labels $\widebar{y}_1$ and $\widebar{y}_2$, the training loss will be incurred from $\widebar{y}_1$ and $\widebar{y}_2$ independently. When the algorithm learns with $\widebar{y}_1$, it will assume that $\widebar{y}_2$ has some possibility to be the correct label. Similarly, if the algorithm learns with $\widebar{y}_2$, it will also assume that $\widebar{y}_1$ has some possibility to be the correct label. However, if the algorithm can learn from both $\widebar{y}_1$ and $\widebar{y}_2$ simultaneously, the supervision information would be richer as the correct label cannot be neither $\widebar{y}_1$ nor $\widebar{y}_2$.
%given the shuffled training set with MCLs $\widebar{\mathcal{D}} = \{(\boldsymbol{x}_i,\widebar{Y}_i)\}_{i=1}^n$, in each mini-batch, we decompose $\widebar{Y}_i$
\subsection{Unbiased Risk Estimator}\label{risk_estimator}
The above example has shown that the supervision information is diluted after decomposition. The basic reason lies in that ordinary complementary-label learning approaches are designed by only considering the data distribution with a single CL, i.e., $\widebar{p}(\boldsymbol{x},\widebar{y})$. However, the data distribution with MCLs $\widebar{p}(\boldsymbol{x},\widebar{Y})$ becomes much different, and the wrappers fail to capture such distribution because they do not treat MCLs as a whole for each example. To solve this problem, %we first explicitly formulate the data distribution with MCLs, and then
we propose an unbiased estimator of the original classification risk for learning with MCLs as a whole.

%With the assumed data distribution $\widebar{p}(\boldsymbol{x},\widebar{Y})$, we will derive an unbiased risk estimator of learning with MCLs. Firstly,
We first relate the data distribution with ordinary labels to that with MCLs by the following lemma.
\begin{lemma}
\label{lemma2}
The following equality holds:
\begin{gather}
\nonumber
p(y\mid\boldsymbol{x}) = 1 - \sum\nolimits_{j = 1}^{k-1} \Big(\frac{k-1}{j}\sum\nolimits_{\widebar{Y}\in\widebar{\mathcal{Y}}_{j}^y}\widebar{p}(\widebar{Y}, s=j\mid \boldsymbol{x})\Big),
\end{gather}
where $\widebar{\mathcal{Y}}_j^y$ is the set of all the possible label sets with size $j$ that include a specific label $y\in\mathcal{Y}$, i.e.,
\begin{gather}
\nonumber
\widebar{\mathcal{Y}}_j^y := \{\widebar{Y}\in\widebar{\mathcal{Y}}\ |\ y\in\widebar{Y}, |\widebar{Y}|=j\}.
\end{gather}
\end{lemma}
The proof is provided in Appendix B.1.

Based on Lemma~\ref{lemma2}, we derive an unbiased estimator of the ordinary classification risk Eq.~(\ref{risk}) by the following theorem.
\begin{theo}
\label{general_risk}
The ordinary classification risk Eq.~(\ref{risk}) can be equivalently expressed as
\begin{gather}
\label{general_unbiased}
R(f) = \sum\nolimits_{j=1}^{k-1} p(s=j)\widebar{R}_j(f),
\end{gather}
where
\begin{gather}
\label{simplified_unbiased}
\widebar{R}_j(f) := \mathbb{E}_{\widebar{p}(\boldsymbol{x},\widebar{Y} \mid s=j)}[\widebar{\mathcal{L}}_j\big(f(\boldsymbol{x}),\widebar{Y}\big)],
\end{gather}
and
\begin{align}
\nonumber
\widebar{\mathcal{L}}_j\big(f(\boldsymbol{x}),\widebar{Y}\big) &:= \sum\nolimits_{y\notin\widebar{Y}}\mathcal{L}\big(f(\boldsymbol{x}),y\big) \\
\label{unbiased_loss}
&\quad-\frac{k-1-j}{j}\sum\nolimits_{y^\prime\in\widebar{Y}}\mathcal{L}\big(f(\boldsymbol{x}),y^\prime\big).
\end{align}
\end{theo}
The proof is provided in Appendix B.2.

It is easy to verify that Eq.~(\ref{general_unbiased}) reduces to Eq.~(\ref{free}) when $p(s=1)=1$. Which means, our approach is a generalization of~\citet{Ishida2019Complementary}. Furthermore, according to Corollary 2 in \citet{Ishida2019Complementary}, our approach is also a generalization of~\citet{ishida2017learning}.

Given the dataset with MCLs $\widebar{\mathcal{D}} = \{(\boldsymbol{x}_i,\widebar{Y}_i)\}_{i=1}^n$, we can empirically approximate $p(s=j)$ by ${n_j}/{n}$ where $n_j$ denotes the number of examples whose complementary label set size is $j$. By further taking into account Eqs.~(\ref{general_unbiased})-(\ref{unbiased_loss}), we can obtain the following empirical approximation of the unbiased risk estimator introduced in Theorem~\ref{general_risk}:
%Similarly, based on Theorem~\ref{general}, using the dataset with MCLs $\widebar{\mathcal{D}} = \{(\boldsymbol{x}_i,\widebar{Y}_i)\}_{i=1}^n$ where $|\widebar{Y}_i|$ may be different for each $i$, we can obtain the following empirical approximation of the unbiased risk estimator introduced in Theorem~\ref{general}:
\begin{align}
\nonumber
\widehat{R}(f)&=\frac{1}{n}\sum\nolimits_{i=1}^n\Big(\sum\nolimits_{y\notin\widebar{Y}_i}\mathcal{L}\big(f(\boldsymbol{x}_i),y\big)\\
\label{emp_risk}
&\quad\quad
-\frac{k-1-|\widebar{Y}_i|}{|\widebar{Y}_i|}\sum\nolimits_{y^\prime\in\widebar{Y}_i}\mathcal{L}\big(f(\boldsymbol{x}_i),y^\prime\big)\Big).
\end{align}
%This case is conceivable when the size of the proposed subset of labels $\widebar{Y}_i$ varies from worker to worker in the crowd-sourcing example.
\paragraph{Estimation Error Bound.}\label{estimation_error}
Here, we derive an estimation error bound for the proposed unbiased risk estimator based on \emph{Rademacher complexity}~\cite{bartlett2002rademacher}.
% Firstly, let us fix the size of the complementary label set $s$ for simplicity and obtaining clear insights.
% \begin{definition}[Redemacher complexity~\cite{bartlett2002rademacher}]
% Let $Z_1,\dots,Z_n$ be $n$ i.i.d. random variables drawn from a probability distribution $\mathcal{D}$, $\mathcal{H}=\{h:\mathcal{Z}\rightarrow\mathbb{R}\}$ be a class of measurable functions. Then the expected Rademacher complexity of $\mathcal{F}$ is defined as
% \begin{gather}
% \nonumber
% \mathfrak{R}_n(\mathcal{H})=\mathbb{E}_{Z_1,\dots,Z_n\sim\mathcal{D}}\mathbb{E}_{\boldsymbol{\sigma}}\bigg[\sup_{h\in\mathcal{H}}\frac{1}{n}\sum_{i=1}^n\sigma_ih(Z_i)\bigg],
% \end{gather}
% where $\boldsymbol{\sigma}=(\sigma_1,\dots,\sigma_n)$ are Rademacher variables taking the value from $\{-1,+1\}$ with even probabilities.
% \end{definition}
Let $\mathcal{F}\subset\{f:\mathbb{R}^d\rightarrow\mathbb{R}^k\}$ be the hypothesis class, $\widehat{f}:=\argmin_{f\in\mathcal{F}}\widehat{R}(f)$ be the empirical risk minimizer, and $f^\star=\argmin_{f\in\mathcal{F}} R(f)$ be the true risk minimizer. Besides, we define the functional space $\mathcal{G}_y$ for the label $y\in\mathcal{Y}$ as
$\mathcal{G}_y = \{g:\boldsymbol{x}\rightarrow f_y(\boldsymbol{x})\ |\ f\in\mathcal{F}\}$. Then, we have the following theorem.
% Then, we have the following theorem.
% \begin{theo}
% \label{error_bound}
% Assume the size of complementary label set $s$ is fixed at $j$ (i.e., p(s=j)=1) and the loss function $\mathcal{L}(\boldsymbol{z},y)$ is $\rho$-Lipschitz with respect to $\boldsymbol{z}$ $(0<\rho<\infty)$ for all $y \in \mathcal{Y}$. %and all the functions in the model class $\mathcal{F}$ are bounded, i.e., there exist a constant $C_b$ such that $\forall y\in\mathcal{Y}, g\in\mathcal{G}_y$, $\left\|g_y\right\|_{\infty}\leq C_b$.
% Let $C_{\mathcal{L}} = \sup_{\boldsymbol{x} \in \mathcal{X}, f \in \mathcal{F}, y\in\mathcal{Y}}\mathcal{L}(f(\boldsymbol{x}), y)$. Then, for any $\delta>0$, with probability at least $1-\delta$,
% \begin{align}
% \nonumber
% R(\widehat{f})- & R(f^\star) \leq \frac{4\rho(k-1)}{j}\sum\nolimits_{y=1}^k\mathfrak{R}_n(\mathcal{G}_y)+\frac{C_j}{\sqrt{n}},
% \end{align}
% where $C_j=(4k-4j-2)C_{\mathcal{L}}\sqrt{\frac{\log\frac{2}{\delta}}{2}}$.
% \end{theo}
% The proof is provided in Appendix C.1.
%Finally, we introduce the general version when $s$ is not fixed.
\begin{theo}\label{error_bound}%\label{unfix_error_bound}
Assume the loss function $\mathcal{L}(f(\boldsymbol{x}),y)$ is $\rho$-Lipschitz with respect to $f(\boldsymbol{x})$ $(0<\rho<\infty)$ for all $y \in \mathcal{Y}$. %and all the functions in the model class $\mathcal{F}$ are bounded, i.e., there exist a constant $C_b$ such that $\forall y\in\mathcal{Y}, g\in\mathcal{G}_y$, $\left\|g_y\right\|_{\infty}\leq C_b$.
Let $C_{\mathcal{L}} = \sup_{\boldsymbol{x} \in \mathcal{X}, f \in \mathcal{F}, y\in\mathcal{Y}}\mathcal{L}(f(\boldsymbol{x}), y)$ and $\mathfrak{R}_n(\mathcal{G}_y)$ be the Rademacher complexity of $\mathcal{G}_y$ given the sample size $n$. Then, for any $\delta>0$, with probability at least $1-\delta$,
\begin{align}
\nonumber
&R(\widehat{f}) - R(f^\star) \\
\nonumber
& \leq \sum\limits_{j=1}^{k-1}p(s=j)\Big(\frac{4\sqrt{2}\rho k(k-1)}{j}\sum\limits_{y=1}^k\mathfrak{R}_{n_j}(\mathcal{G}_y)+\frac{C_j}{\sqrt{n_j}}\Big),
\end{align}
where $C_j=(4k-4j-2)C_{\mathcal{L}}\sqrt{\frac{\log\frac{2(k-1)}{\delta}}{2}}$ for all $j \in \{1, \ldots, k-1\}$ and $n_j$ denotes the number of examples whose complementary label set size is $j$.
\end{theo}
The definition of Rademacher complexity and the proof of Theorem~\ref{error_bound} are provided in Appendix C. Theorem~\ref{error_bound} shows that the empirical risk minimizer converges to the true risk minimizer with high probability as the number of training data approaches infinity. It is
worth noting that this bound is not only related to the Redemacher
complexity of the function class, but also $s$ and
$k$. This observation accords
with our intuition that the learning task will be harder if the
number of classes $k$ increases or the size of the complementary label set $s$ decreases. %Besides, if $j$ increases and $p(s=j)$ increases, the estimation error bound would be tighter, since more supervision information are provided.

%In addition, since the Rademacher complexity $\mathfrak{R}(\mathcal{G}_y)$ can be bounded by $\frac{C_{\mathcal{G}}}{n}$ for a positive constant $C_{\mathcal{G}}$. \\

%\subsection{The Simplified Case of Complementary Label Sets}\label{simplified_case}

%For the generation process of the examples with MCLs, let us assume that the size of the complementary label set $s$ is sampled from a distribution $p(s)$. For the convenience we define the set of the subset of the label space $\mathcal{Y}$ whose size is $s$ asNote that we will relax the constraint on the fixed size of each complementary label set in Section~\ref{general_case}. When the size of each complementary label set is fixed, we assume that all the training examples $\widebar{\mathcal{D}}_s=\{(\boldsymbol{x}_i,\widebar{Y}_i)\}_{i=1}^{n_s}$ are drawn independently and identically from the joint distribution over $\mathbb{R}^d\times\widebar{\mathcal{Y}}_s$ given as follows:
\subsection{Practical Implementation}\label{implement}
In this section, we present the practical implementation of our proposed formulation and improvements of the used loss functions. As described above, we have provided a general unbiased risk estimator that is able to use arbitrary loss functions. There arises a question: Can all loss functions work well in our approach? Unfortunately, the answer is negative.

The original classification risk estimator in Eq.~(\ref{risk}) includes an expectation over a non-negative loss $\mathcal{L}:\mathbb{R}^k\times [k]\rightarrow\mathbb{R}_+$, hence the expected risk and the empirical approximation are both lower-bounded by zero. However, our proposed risk estimator in Theorem \ref{general_risk} contains a negative term. %Although they are still non-negative and unbiased by definition, as a result of the negative terms,
Although the expected risk estimator is unbiased,
the empirical estimator may become unbounded below if the used loss function is unbounded, thereby leading to over-fitting. Similar issues have also been shown by~\citet{Kiryo2017Positive,Ishida2019Complementary}.
The above analysis suggests that a bounded loss is probably better than an unbounded loss, in our empirical risk estimator (i.e., Eq.~(\ref{emp_risk})).

To demonstrate the above conjecture, we would like to insert bounded and unbounded losses into Eq.~(\ref{emp_risk}), for comparison studies.
%conduct experiments to compare four bounded losses with the widely-used unbounded categorical cross entropy loss. The experimental results will be reported in the next section.
Note that we assume that the softmax function is absorbed in each loss, and denote by $p_{\boldsymbol{\theta}}(y|\boldsymbol{x})={\exp(f_y(\boldsymbol{x}))}/{(\sum_{j=1}^k\exp(f_j(\boldsymbol{x})))}$ the predicted probability of the instance $\boldsymbol{x}$ belonging to class $y$, where $\boldsymbol{\theta}$ denotes the parameters of the model $f$. In this way, we list the compared loss functions as follows.
\begin{figure*}[!t]
%\addtolength{\belowcaptionskip}{-10pt}
\centering
\subfigure[MNIST, linear]{
    %\label{fig:mini:subfig:a} %% label for first subfigure
      \includegraphics[width=1.6in]{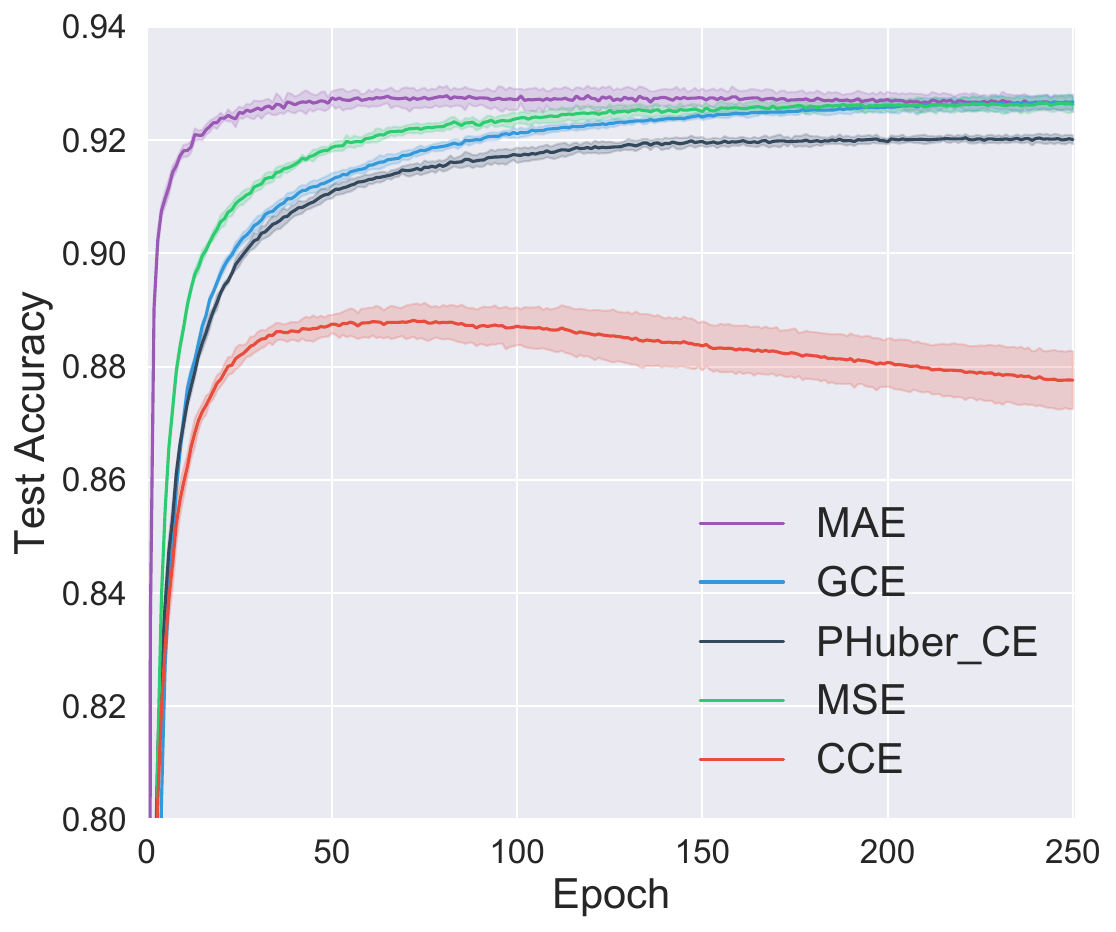}}%
  \subfigure[MNIST, MLP]{
    %\label{fig:mini:subfig:b} %% label for second subfigure
      \includegraphics[width=1.6in]{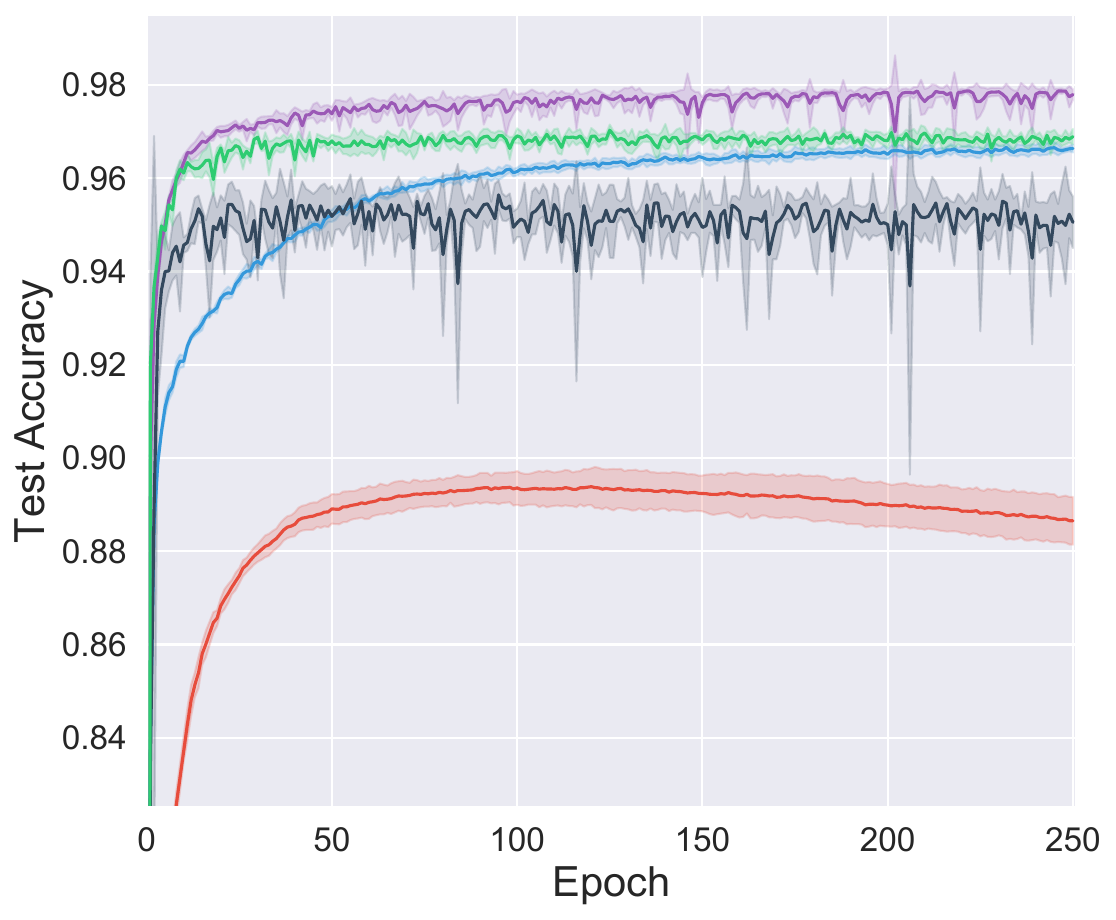}}
      \subfigure[Fashion MNIST, linear]{
    %\label{fig:mini:subfig:c} %% label for second subfigure
      \includegraphics[width=1.6in]{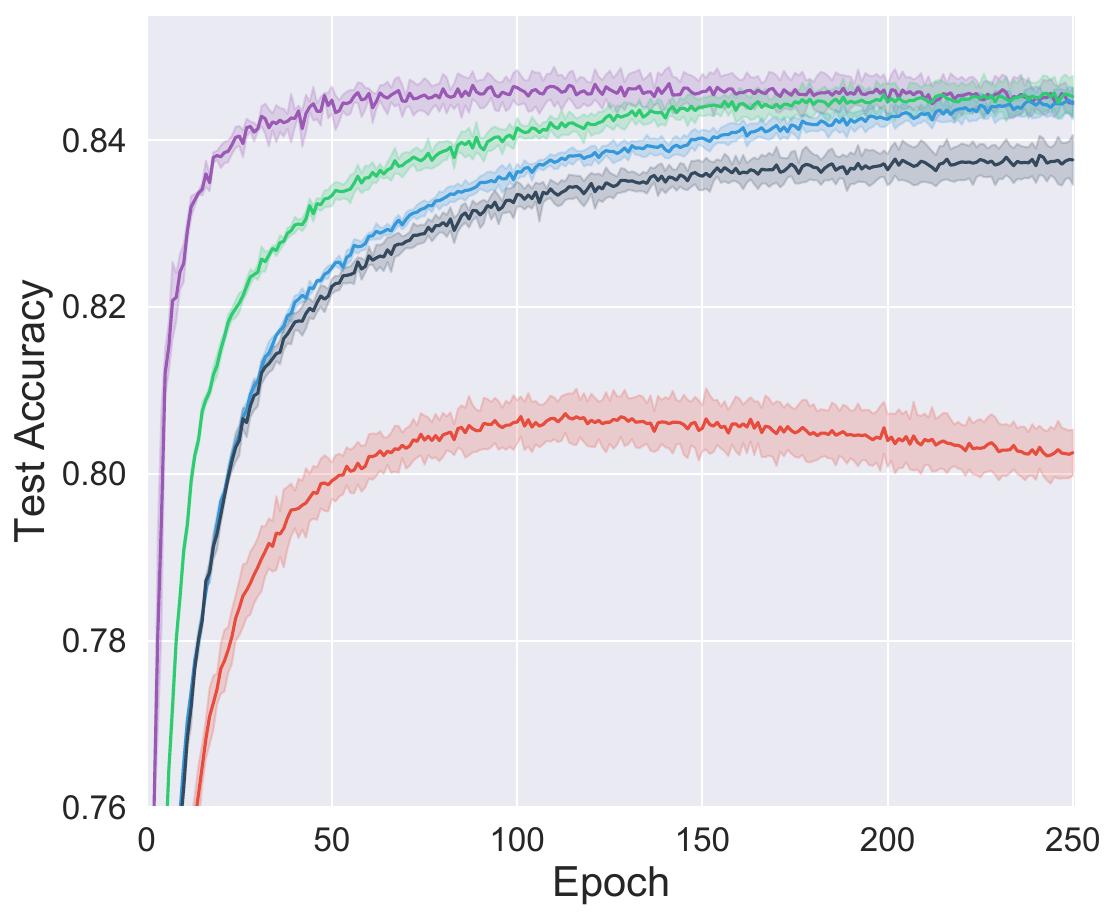}}
      \subfigure[Fashion MNIST, MLP]{
    %\label{fig:mini:subfig:d} %% label for second subfigure
      \includegraphics[width=1.6in]{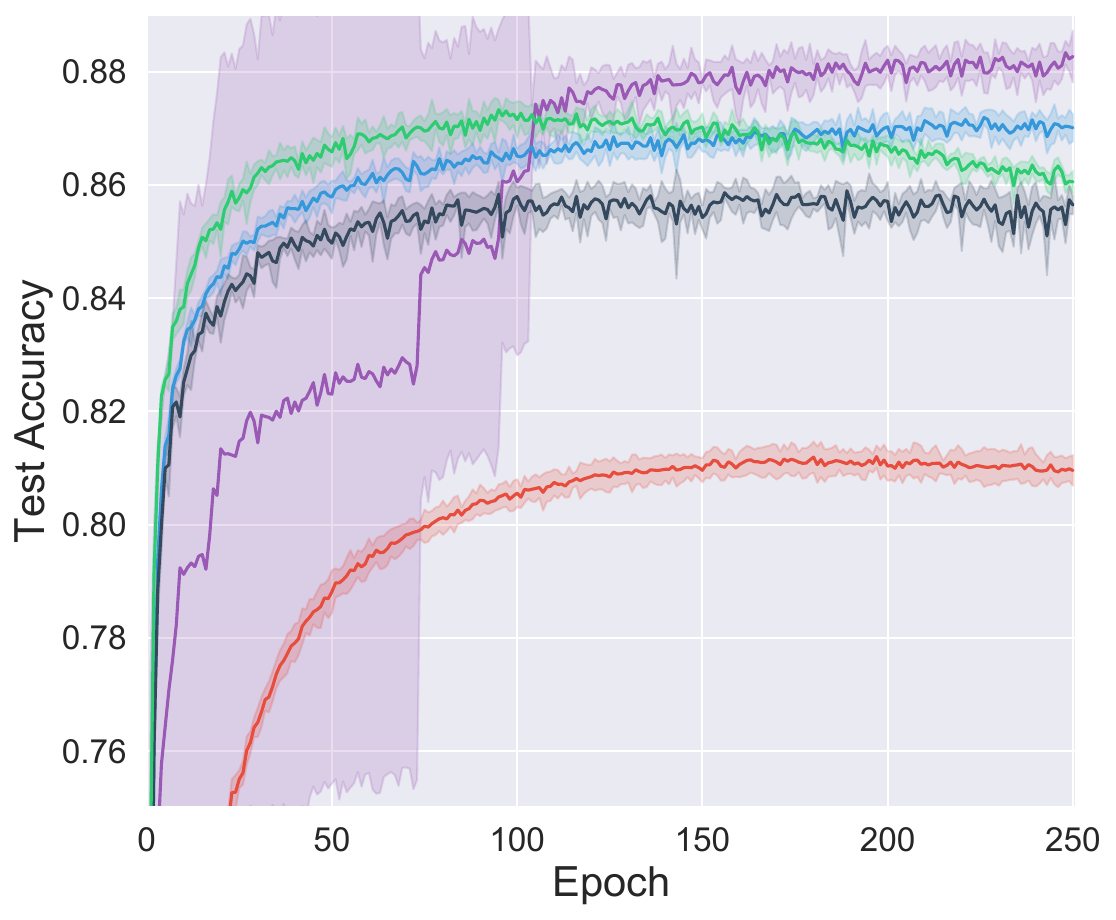}}
      \vspace{-0.3cm}
      \subfigure[Kuzushiji MNIST, linear]{
    %\label{fig:mini:subfig:d} %% label for second subfigure
      \includegraphics[width=1.6in]{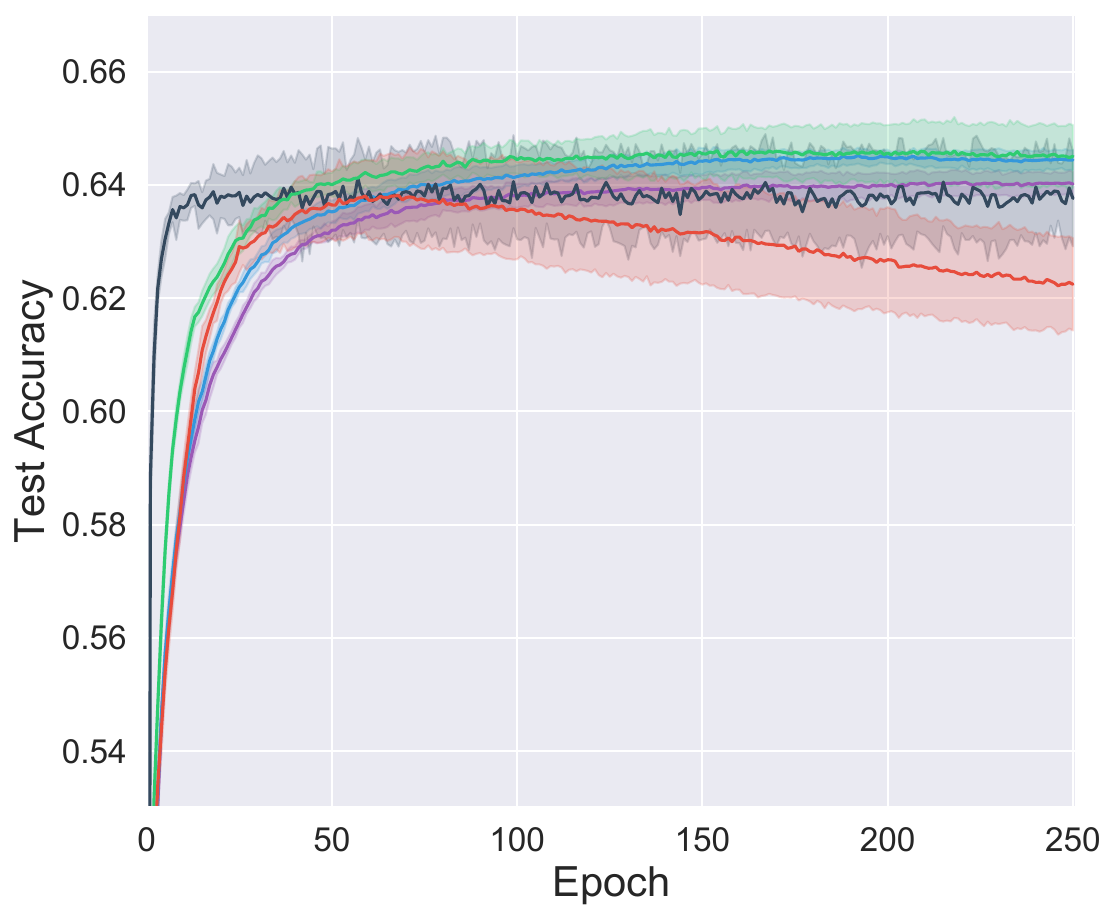}}
      \subfigure[Kuzushiji MNIST, MLP]{
    %\label{fig:mini:subfig:d} %% label for second subfigure
      \includegraphics[width=1.6in]{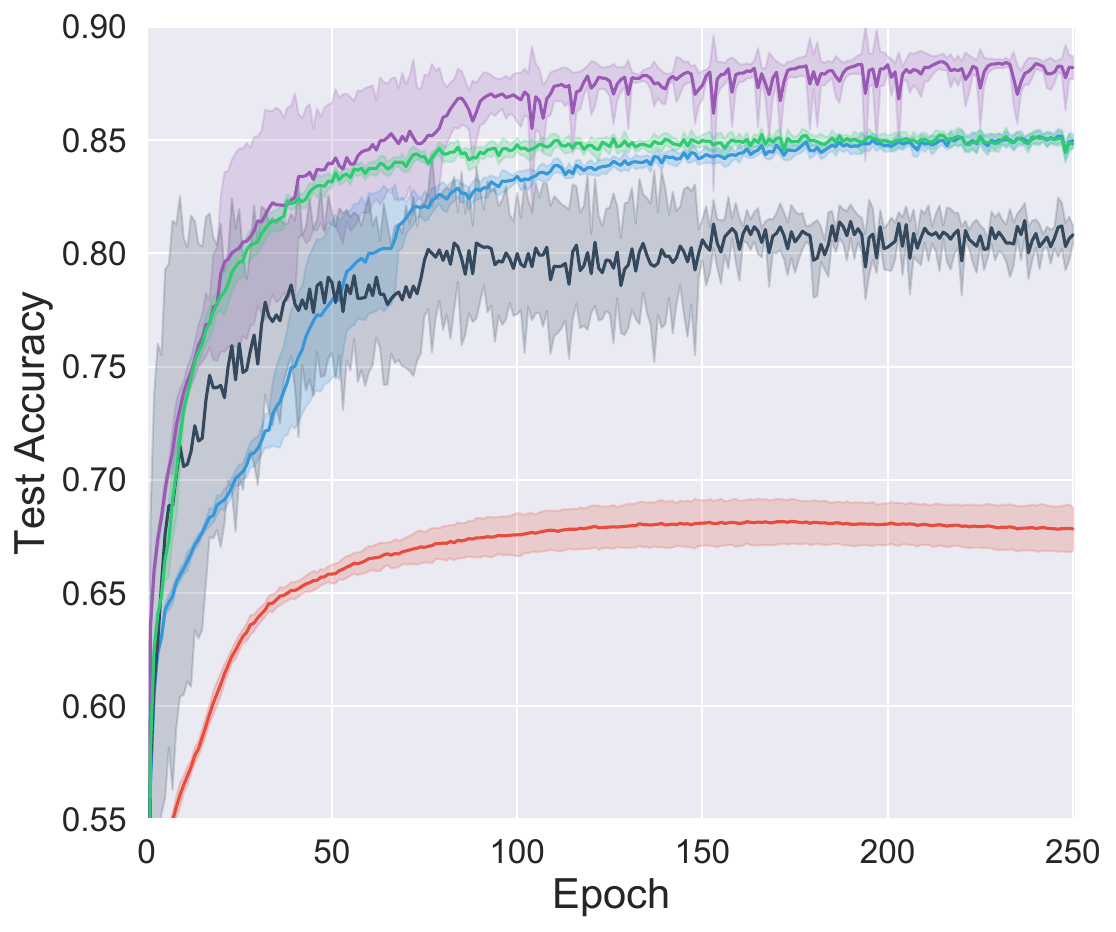}}
      \subfigure[CIFAR-10, ResNet]{
    %\label{fig:mini:subfig:d} %% label for second subfigure
      \includegraphics[width=1.6in]{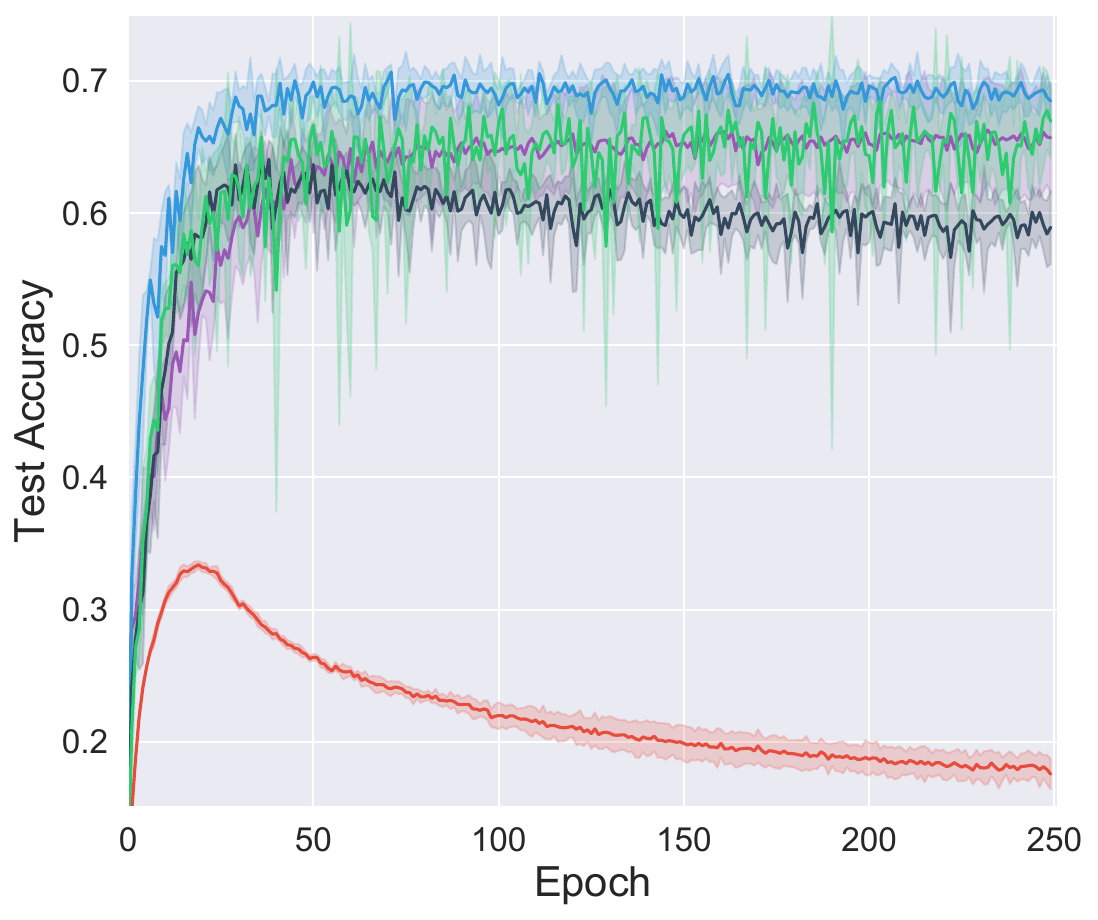}}
      \subfigure[CIFAR-10, DenseNet]{
    %\label{fig:mini:subfig:d} %% label for second subfigure
      \includegraphics[width=1.6in]{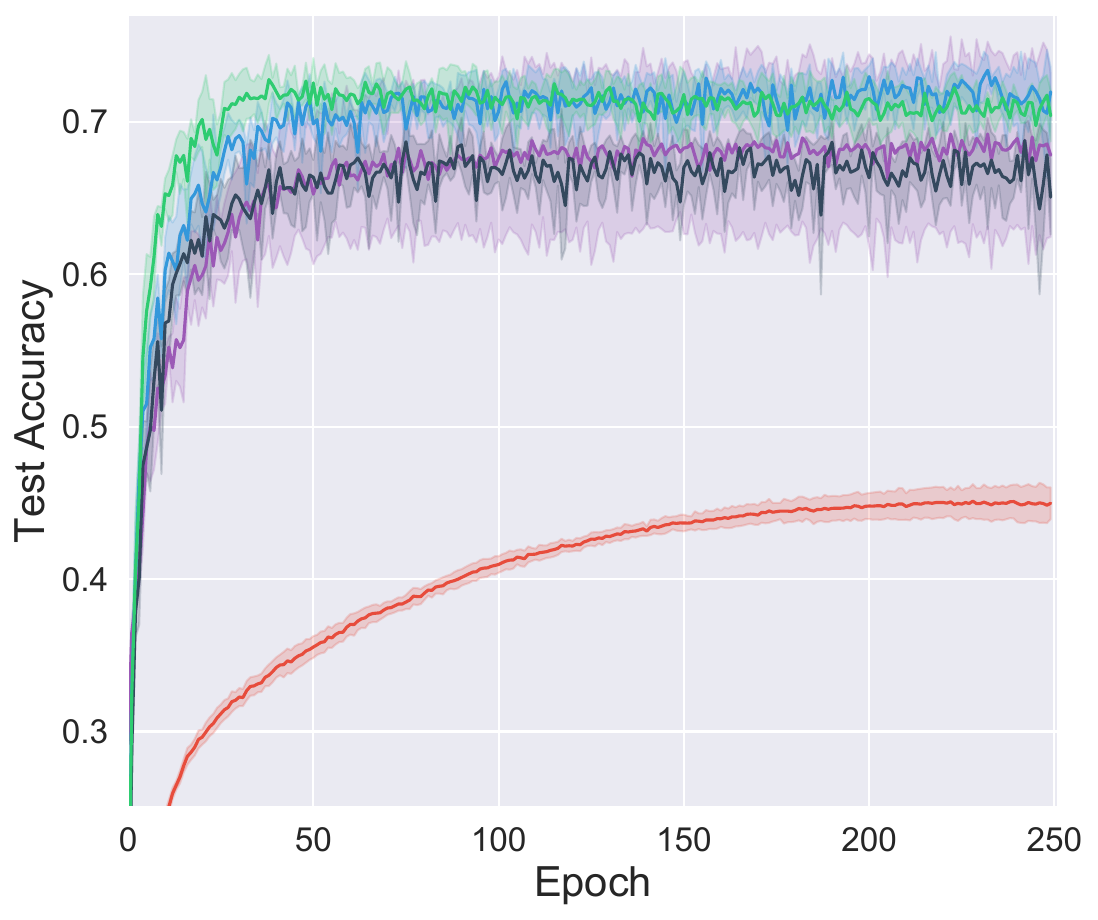}}
  \caption{Experimental results of different loss functions for different datasets and models. Dark colors show the mean accuracy of 5 trials and light colors show the standard deviation.}
  \vspace{-0.3cm}
  \label{fig1} %% label for entire figure
\end{figure*}
\begin{itemize}
\setlength{\parskip}{-5pt}
\item Categorical Cross Entropy (CCE):
\begin{gather}
\nonumber
\mathcal{L}_{\text{CCE}}(f(\boldsymbol{x}),y)=-\log p_{\boldsymbol{\theta}}(y|\boldsymbol{x}).
\end{gather}
\item Mean Absolute Error (MAE):
\begin{gather}
\nonumber
\mathcal{L}_{\text{MAE}}(f(\boldsymbol{x}),y)=2-2p_{\boldsymbol{\theta}}(y|\boldsymbol{x}).
\end{gather}
\item Mean Square Error (MSE):
\begin{gather}
\nonumber
\mathcal{L}_{\text{MSE}}(f(\boldsymbol{x}),y)=1-2p_{\boldsymbol{\theta}}(y|\boldsymbol{x})+\sum\nolimits_{j=1}^kp_{\boldsymbol{\theta}}(j|\boldsymbol{x})^2.
\end{gather}
\item Generalized Cross Entropy (GCE)~\cite{zhang2018generalized}:
\begin{gather}
\nonumber
\mathcal{L}_{\text{GCE}}(f(\boldsymbol{x}),y)={(1-p_{\boldsymbol{\theta}}(y|\boldsymbol{x})^q)}/{q},
\end{gather}
where $q\in(0,1]$ is a user-defined hyper-parameter. We set $q=0.7$, as suggested by \citet{zhang2018generalized}. %as this the recommended setting by \citet{zhang2018generalized}.
%We use the recommended setting $q=0.7$. %, as recommended by \citet{zhang2018generalized}.%this is the recommended setting by
\item Partially Huberised Cross Entropy (PHuber-CE) \cite{menon2020can}:
\begin{align}
\small
\nonumber
\mathcal{L}_{\text{PHuber-CE}}(f(\boldsymbol{x}),y)=\left\{\begin{matrix}
-\log p_{\boldsymbol{\theta}}(y|\boldsymbol{x}), \text{if }p_{\boldsymbol{\theta}}(y|\boldsymbol{x})\geq\frac{1}{\tau},\\
-\tau p_{\boldsymbol{\theta}}(y|\boldsymbol{x})+\log\tau+1,\ \text{else},
\end{matrix}\right.
\end{align}
where $\tau>0$ is a user-defined hyper-parameter. We set $\tau=10$, because it works well in \citet{menon2020can}.
\end{itemize}
The detailed derivations of the above loss functions and their bounds are provided in Appendix D. Among these losses, CCE is unbounded while the others are bounded. We will experimentally demonstrate (Figure~\ref{fig1}) that
%Our experimental results (Figure~\ref{fig1}) clearly demonstrate
by inserting the above losses into Eq.~(\ref{emp_risk}), bounded loss is significantly better than unbounded loss. Furthermore, we conduct a deeper analysis of MAE because MAE has the special property that MAE is not only bounded, but also satisfies the symmetric condition~\cite{ghosh2017robust}, i.e.,
$\sum_{y=1}^k\mathcal{L}_{\text{MAE}}\big(f(\boldsymbol{x}),y\big)=2k-2$, which means the sum of the losses over all classes is a constant for arbitrary examples.
However, is MAE good enough? Previous studies~\cite{zhang2018generalized,wang2019improving} have already shown that MAE suffers from the optimization issue, which would affect its practical performance.
To alleviate this problem, we further improve MAE by proposing two upper-bound surrogate loss functions. Specifically, by using MAE in Eq.~(\ref{emp_risk}), we obtain
%\widebar{\mathcal{L}}_{\text{MAE}}\big(f(\boldsymbol{x}_i),\widebar{Y}_i\big)
\begin{align}
\nonumber
\widehat{R}(f) =& \frac{k-1}{|\widebar{Y}_i|}\sum\nolimits_{y\notin\widebar{Y}_i}\mathcal{L}_{\text{MAE}}\big(f(\boldsymbol{x}_i),y\big)\\
\label{mae_loss}
=&\frac{2k-2}{|\widebar{Y}_i|}\mathcal{L}^\prime_{\text{MAE}}\big(f(\boldsymbol{x}_i),\widebar{Y}_i\big)+Z_i,
\end{align}
where $\mathcal{L}^\prime_{\text{MAE}}\big(f(\boldsymbol{x}_i),\widebar{Y}_i\big):=1-\sum_{j\notin\widebar{Y}_i}p_{\boldsymbol{\theta}}(j|\boldsymbol{x}_i)$, and $Z_i$ is a constant independent of $f(\boldsymbol{x}_i)$. It is clear that minimizing $\mathcal{L}^\prime_{\text{MAE}}\big(f(\boldsymbol{x}_i),\widebar{Y}_i\big)$ is equivalent to minimizing $\sum_{y\notin\widebar{Y}_i}\mathcal{L}_{\text{MAE}}\big(f(\boldsymbol{x},y)\big)$.

%We demonstrate this issue with an illustrative numerical example by comparing MAE and CCE. In figure 1, we show an example of training a linear model trained on the handwritten digits dataset MNIST, with complementary labels generated to satisfy the uniform distribution. We use Adam for optimization with learning rate set to $10^{-3}$, weight decay set to $10^{-5}$, mini-batch size set to 256, and epoch number set to 250. As can be seen from Figure 1, when we use CCE in our formulation, the test accuracy first increases and then gradually decreases. This issue is more serious when we use a more flexible model. In contrast, when we use MAE, the performance is quite stable.
%In a later section, we will conduct more experiments to compare MAE with more bounded loss functions, and we find that MAE still achieves comparable performance over the other bounded loss functions. We suspect that the superiority of MAE is because MAE satisfies the \text{symmetric condition}, i.e.,
Based on this fact, we further introduce two upper-bound surrogate loss functions of $\mathcal{L}^\prime_{\text{MAE}}$:
\begin{align}
\nonumber
\mathcal{L}_{\text{EXP}}(f(\boldsymbol{x}_i),\widebar{Y}_i) &= \exp\Big(-\sum\nolimits_{j\notin\widebar{Y}_i}p_{\boldsymbol{\theta}}(j|\boldsymbol{x}_i)\Big),\\
\nonumber
\mathcal{L}_{\text{LOG}}(f(\boldsymbol{x}_i),\widebar{Y}_i) &= -\log\Big(\sum\nolimits_{j\notin\widebar{Y}_i}p_{\boldsymbol{\theta}}(j|\boldsymbol{x}_i)\Big).
\end{align}
One can easily verify that $\mathcal{L}_{\text{MAE}}^\prime$ is upper bounded by $\mathcal{L}_{\text{EXP}}$ and $\mathcal{L}_{\text{LOG}}$ using the two inequalities $1-z\leq\exp(-z)$ and $1-z\leq -\log z$, respectively. By replacing $\mathcal{L}_{\text{MAE}}^\prime$ by $\mathcal{L}_{\text{LOG}}$ and $\mathcal{L}_{\text{LOG}}$ in Eq.~(\ref{mae_loss}), we obtain two new methods for learning with MCLs.
We explain the advantage of $\mathcal{L}_{\text{EXP}}$ and $\mathcal{L}_{\text{LOG}}$ over $\mathcal{L}_{\text{MAE}}^\prime$ by closely examining their gradients:
\begin{align}
\nonumber
\frac{\partial\mathcal{L}^\prime_{\text{MAE}}}{\partial \boldsymbol{\theta}} &= \left\{\begin{matrix}
-\nabla_{\boldsymbol{\theta}}p_{\boldsymbol{\theta}}(j|\boldsymbol{x}_i),&\ \text{if}\ j\notin\widebar{Y}_i,\\
0,&\ \text{else},
\end{matrix}\right.\\
\nonumber
\frac{\partial\mathcal{L}_{\text{EXP}}}{\partial \boldsymbol{\theta}} &= \left\{\begin{matrix}
-\nabla_{\boldsymbol{\theta}}p_{\boldsymbol{\theta}}(j|\boldsymbol{x}_i)\cdot w_{\text{EXP}},&\ \text{if}\ j\notin\widebar{Y}_i,\\
0,&\ \text{else},
\end{matrix}\right.\\
\nonumber
\frac{\partial\mathcal{L}_{\text{LOG}}}{\partial \boldsymbol{\theta}} &= \left\{\begin{matrix}
-\nabla_{\boldsymbol{\theta}}p_{\boldsymbol{\theta}}(j|\boldsymbol{x}_i)\cdot w_{\text{LOG}},&\ \text{if}\ j\notin\widebar{Y}_i,\\
0,&\ \text{else},
\end{matrix}\right.
\end{align}
where $w_{\text{EXP}}=\exp\big(-\sum_{j\notin\widebar{Y}_i}p_{\boldsymbol{\theta}}(j|\boldsymbol{x}_i)\big)$ and $w_{\text{LOG}}=\big(\sum_{j\notin\widebar{Y}_i}p_{\boldsymbol{\theta}}(j|\boldsymbol{x}_i)\big)^{-1}$. From their gradients, we can clearly observe that $\mathcal{L}_{\text{MAE}}^\prime$ basically treats each example equally, while $\mathcal{L}_{\text{EXP}}$ and $\mathcal{L}_{\text{LOG}}$ give more weights to difficult examples. Concretely, if $\sum_{j\notin\widebar{Y}_i}p_{\boldsymbol{\theta}}(j|\boldsymbol{x}_i)$ is small, both $w_{\text{EXP}}$ and $w_{\text{LOG}}$ would be large. In other words, $\mathcal{L}_{\text{EXP}}$ and $\mathcal{L}_{\text{LOG}}$ pay more attention to hard examples whose sum of the predicted confidences of all the non-complementary labels is small.
\section{Experiments}

\begin{table*}[t]
\centering
\caption{Classification accuracy (mean$\pm$std) of each algorithm on the four UCI datasets using a linear model for 5 trials. The best performance among all the approaches is highlighted in boldface. In addition, $\bullet/\circ$ indicates whether the performance of our approach (the best of EXP and LOG) is statistically superior/inferior to the comparing algorithm on each dataset (paired $t$-test at 0.05 significance level).}
\label{uci}
\resizebox{1.00\textwidth}{!}{
\setlength{\tabcolsep}{5mm}{
\begin{tabular}{c|c|c|c|c|c}
\toprule
\multicolumn{2}{c|}{Approach}              & Yeast & Texture & Dermatology & Synthetic Control\\
\midrule
\multirow{2}{*}{Upper-bound Losses }
                              & EXP       &  54.94$\pm$1.56\%$\bullet$    &   97.51$\pm$0.09\%$\bullet$ & 98.89$\pm$0.37\% & 27.87$\pm$5.13\%$\bullet$
                         \\
                              & LOG &   \bf{60.11$\pm$1.93\%}  &  98.88$\pm$0.43\% & \bf{99.46$\pm$1.14\%}    &  \bf{90.73$\pm$4.41\%}   \\
\midrule
\multirow{4}{*}{Bounded Losses}
& MAE     &  33.07$\pm$0.37\%$\bullet$ & 85.29$\pm$7.93\%$\bullet$ & 85.39$\pm$2.58\%$\bullet$ & 23.50$\pm$2.44\%$\bullet$ \\
                              & MSE  &  58.17$\pm$1.52\%$\bullet$ & 97.59$\pm$0.16\%$\bullet$     &  97.84$\pm$1.21\%$\bullet$ & 34.20$\pm$8.69\%$\bullet$ \\
                              & GCE       &  57.56$\pm$1.56\%$\bullet$ & 97.25$\pm$0.31\%$\bullet$ & 97.53$\pm$1.81\%$\bullet$ & 23.67$\pm$3.10\%$\bullet$ \\
                              & Phuber-CE & 55.54$\pm$1.03\%$\bullet$ & 94.89$\pm$3.28\%$\bullet$ & 95.14$\pm$2.41\%$\bullet$ & 24.71$\pm$3.18\%$\bullet$\\
\midrule
Unbounded Loss                & CCE & 49.50$\pm$3.58\%$\bullet$ & 92.08$\pm$1.15\%$\bullet$ & 83.19$\pm$3.65\%$\bullet$ & 63.47$\pm$6.91\%$\bullet$\\
\midrule
\multirow{5}{*}{Decomposition before Shuffle}  & GA        &  27.91$\pm$5.02\%$\bullet$ & 90.93$\pm$1.34\%$\bullet$ & 36.05$\pm$9.79\%$\bullet$ & 18.12$\pm$1.74\%$\bullet$\\
                              & NN   &  32.73$\pm$3.59\%$\bullet$ & 96.29$\pm$0.39\%$\bullet$ & 61.49$\pm$6.83\%$\bullet$ & 55.12$\pm$4.43\%$\bullet$   \\
                              & FREE      &   35.50$\pm$2.79\%$\bullet$ & 94.36$\pm$1.08\%$\bullet$ & 86.30$\pm$5.62\%$\bullet$ & 76.95$\pm$3.26\%$\bullet$ \\
                              & PC        &   53.89$\pm$3.53\%$\bullet$ & 92.68$\pm$0.81\%$\bullet$ & 96.27$\pm$3.07\%$\bullet$ & 72.63$\pm$5.86\%$\bullet$  \\
                              & Forward   &  58.15$\pm$1.54\%$\bullet$ & \bf{98.95$\pm$0.17\%} & 99.37$\pm$0.85\% & 38.77$\pm$6.06\%$\bullet$  \\
\midrule
\multirow{5}{*}{Decomposition after Shuffle}  & GA        &   28.21$\pm$1.53\%$\bullet$ & 83.66$\pm$2.27\%$\bullet$ & 42.05$\pm$7.94\%$\bullet$ & 25.46$\pm$1.28\%$\bullet$\\
                              & NN        &  36.04$\pm$2.24\%$\bullet$ & 93.91$\pm$0.40\%$\bullet$ & 62.54$\pm$9.19\%$\bullet$ & 59.80$\pm$5.14\%$\bullet$\\
                              & FREE      &   43.47$\pm$1.36\%$\bullet$ & 93.94$\pm$0.72\%$\bullet$ & 86.22$\pm$6.07\%$\bullet$ & 73.33$\pm$2.17\%$\bullet$\\
                              & PC     &  54.58$\pm$2.57\%$\bullet$ & 94.19$\pm$1.21\%$\bullet$ & 95.73$\pm$3.33\%$\bullet$ & 69.53$\pm$9.01\%$\bullet$\\
                              & Forward   &  59.46$\pm$1.16\% & 97.65$\pm$0.32\%$\bullet$ & 99.03$\pm$1.33\% & 43.57$\pm$5.83\%$\bullet$  \\
\midrule
Partial Label Convex Formulation                    & CLPL      &  55.39$\pm$1.21\%$\bullet$     &   92.07$\pm$0.88\%$\bullet$    &     99.42$\pm$0.54\%  & 63.57$\pm$5.46\%$\bullet$   \\
\bottomrule
\end{tabular}
}
}
\end{table*}

\begin{table*}[!t]
\vspace{-0.2cm}
\centering\caption{Classification accuracy (mean$\pm$std) of each algorithm on the four benchmark datasets using a linear model for 5 trials. The best performance among all the approaches is highlighted in boldface. In addition, $\bullet/\circ$ indicates whether the performance of our approach (the best of EXP and LOG) is statistically superior/inferior to the comparing algorithm on each dataset (paired $t$-test at 0.05 significance level).}
\label{linear}
%\scriptsize
\resizebox{1.00\textwidth}{!}{
\setlength{\tabcolsep}{5mm}{
\begin{tabular}{c|c|c|c|c|c}
\toprule
\multicolumn{2}{c|}{Approach}              & MNIST & Kuzushiji & Fashion & 20Newsgroups \\
\midrule
\multirow{2}{*}{Upper-bound Losses }
                              & EXP       &   \bf{92.67$\pm$0.11\%}   &   64.23$\pm$0.33\%$\bullet$     &    \bf{84.56$\pm$0.25\%}            & 81.72$\pm$0.39\%$\bullet$\\
                              & LOG &  92.58$\pm$0.09\%$\bullet$     &   \bf{68.89$\pm$0.25\%}     &     84.42$\pm$0.16\% & \bf{84.06$\pm$0.57\%}         \\
\midrule
\multirow{4}{*}{Bounded Losses}
& MAE     &   92.66$\pm$0.12\%    &   64.03$\pm$0.19\%$\bullet$  & 84.50$\pm$0.16\% & 79.68$\pm$1.40\%$\bullet$ \\
                              & MSE       &   92.64$\pm$0.13\%    &    64.51$\pm$0.55\%$\bullet$    &    84.53$\pm$0.20\%     & 81.55$\pm$0.52\%$\bullet$       \\
                              & GCE       &  92.66$\pm$0.12\%     &  64.44$\pm$0.17\%$\bullet$      &       84.44$\pm$0.15\%         & 81.78$\pm$0.60\%$\bullet$\\
                              & Phuber-CE &   92.02$\pm$0.07\%$\bullet$    &   63.81$\pm$0.75\%$\bullet$    &  83.76$\pm$0.22\%$\bullet$    & 73.52$\pm$1.04\%$\bullet$          \\
\midrule
Unbounded Loss                & CCE        &   88.23$\pm$0.19\%$\bullet$    &  62.27$\pm$0.84\%$\bullet$      &        80.25$\pm$0.29\%$\bullet$    & 63.78$\pm$0.79\%$\bullet$   \\
\midrule
\multirow{5}{*}{Decomposition before Shuffle}  & GA        &   85.51$\pm$0.26\%$\bullet$    &    55.61$\pm$0.24\%$\bullet$    &      78.64$\pm$0.33\%$\bullet$   &  76.64$\pm$0.62\%$\bullet$   \\
                              & NN        &   88.09$\pm$0.16\%$\bullet$    &    60.54$\pm$0.23\%$\bullet$    &     80.68$\pm$0.07\%$\bullet$   & 76.00$\pm$0.37\%$\bullet$        \\
                              & FREE      &   89.35$\pm$0.14\%$\bullet$    &  65.21$\pm$0.45\%$\bullet$      &      81.22$\pm$0.11\%$\bullet$   & 68.34$\pm$0.72\%$\bullet$        \\
                              & PC        &  88.21$\pm$0.23\%$\bullet$    &   62.76$\pm$0.40\%$\bullet$     &      80.60$\pm$0.18\%$\bullet$    & 66.91$\pm$1.20\%$\bullet$     \\
                              & Forward   &   92.57$\pm$0.05\%$\bullet$    &   63.51$\pm$0.22\%$\bullet$     &   84.38$\pm$0.20\%   & 74.69$\pm$1.14\%$\bullet$        \\
\midrule
\multirow{5}{*}{Decomposition after Shuffle}  & GA      &   83.16$\pm$0.22\%$\bullet$    &  56.31$\pm$0.42\%$\bullet$   &          73.37$\pm$0.10\%$\bullet$   &  66.14$\pm$0.79\%$\bullet$\\
                              & NN        &  88.79$\pm$0.26\%$\bullet$    &   63.19$\pm$0.12\%$\bullet$     &    79.77$\pm$0.14\%$\bullet$  & 66.35$\pm$0.53\%$\bullet$          \\
                              & FREE      &  89.02$\pm$0.22\%$\bullet$    &   64.18$\pm$0.18\%$\bullet$    &       80.11$\pm$0.04\%$\bullet$   & 66.16$\pm$0.60\%$\bullet$       \\
                              & PC        & 87.76$\pm$0.17\%$\bullet$   &   61.64$\pm$0.38\%$\bullet$     &    80.58$\pm$0.17\%$\bullet$   & 65.64$\pm$0.81\%$\bullet$       \\
                              & Forward   &   92.54$\pm$0.04\%$\bullet$  &   63.69$\pm$0.14\%$\bullet$    &      84.37$\pm$0.17\%$\bullet$    & 71.98$\pm$3.41\%$\bullet$   \\
\midrule
Partial Label Convex Formulation                    & CLPL      &  81.85$\pm$0.27\%$\bullet$     &    55.31$\pm$0.23\%$\bullet$    &     77.26$\pm$0.10\%$\bullet$   & 81.48$\pm$0.45\%$\bullet$   \\
\bottomrule
\end{tabular}
}
}
\vspace{-0.2cm}
\end{table*}
In this section, we conduct extensive experiments to evaluate
the performance of our proposed approaches including the two wrappers, the unbiased risk estimator with various loss functions and the two upper-bound surrogate loss functions.
%\subsection{Experiment Configuration}
\begin{table*}[!t]
\centering
\caption{Classification accuracy (mean$\pm$std) of each algorithm on the five benchmark datasets using neural networks for 5 trials. The best performance among all the approaches is highlighted in boldface. In addition, $\bullet/\circ$ indicates whether the performance of our approach (the best of EXP and LOG) is statistically superior/inferior to the comparing algorithm on each dataset (paired $t$-test at 0.05 significance level).}
\label{deep}
%\scriptsize
\resizebox{1.00\textwidth}{!}{
\setlength{\tabcolsep}{2mm}{
\begin{tabular}{c|c|c|c|c|c|c|c}
\toprule
\multicolumn{2}{c|}{Approach}              & MNIST & Kuzushiji & Fashion & CIFAR-10 R & CIFAR-10 D & 20Newsgroups \\
\midrule
\multirow{2}{*}{Upper-bound Losses }
                              & EXP       &  97.80$\pm$0.06\%     &   \bf{88.25$\pm$0.28\%}     &   88.07$\pm$0.19\%$\bullet$      &   72.49$\pm$0.84\%$\bullet$  & 75.53$\pm$0.58\%
                         & 77.22$\pm$1.22\%\\
                              & LOG &   \bf{97.86$\pm$0.13\%}    &   88.24$\pm$0.08\%     &   \bf{88.36$\pm$0.26\%}      &   \bf{75.38$\pm$0.34\%}     & \bf{75.80$\pm$0.62\%} & \bf{79.46$\pm$0.94\%}\\
\midrule
\multirow{4}{*}{Bounded Losses}
& MAE     &   97.81$\pm$0.04\%    &  88.11$\pm$0.40\%     &   88.13$\pm$0.23\%      &  65.57$\pm$4.08\%$\bullet$ &  68.24$\pm$5.84\%$\bullet$ & 49.83$\pm$4.01\%$\bullet$ \\
                              & MSE       &  96.84$\pm$0.08\%$\bullet$     &    84.97$\pm$0.23\%$\bullet$    &   86.14$\pm$0.04\%$\bullet$    &   63.58$\pm$1.19\%$\bullet$  & 70.89$\pm$0.81\%$\bullet$  & 72.19$\pm$0.59\%$\bullet$  \\
                              & GCE       &   96.62$\pm$0.08\%$\bullet$    &   85.02$\pm$0.26\%$\bullet$     &    87.03$\pm$0.20\%$\bullet$     &    68.40$\pm$1.05\%$\bullet$  & 71.54$\pm$0.83\%$\bullet$ & 74.96$\pm$0.47\%$\bullet$ \\
                              & Phuber-CE &  95.00$\pm$0.36\%$\bullet$    &  80.66$\pm$0.32\%$\bullet$  &  85.52$\pm$0.18\%$\bullet$     &    59.64$\pm$1.21\%$\bullet$  & 66.49$\pm$0.67\%$\bullet$  & 62.63$\pm$2.32\%$\bullet$ \\
\midrule
Unbounded Loss                & CCE        &  88.64$\pm$0.50\%$\bullet$     & 67.86$\pm$1.01\%$\bullet$      & 80.97$\pm$0.23\%$\bullet$        &   18.01$\pm$0.63\%$\bullet$   & 44.94$\pm$1.20\%$\bullet$  & 54.96$\pm$0.38\%$\bullet$ \\
\midrule
\multirow{5}{*}{Decomposition before Shuffle}  & GA        &   96.36$\pm$0.05\%$\bullet$    &   84.35$\pm$0.22\%$\bullet$     &   85.59$\pm$0.30\%$\bullet$     &   69.05$\pm$0.83\%$\bullet$   &  65.38$\pm$1.40\%$\bullet$  & 79.06$\pm$0.57\%\\
                              & NN        &  96.70$\pm$0.08\%$\bullet$     &   82.21$\pm$0.36\%$\bullet$     &  86.29$\pm$0.10\%$\bullet$       &  63.85$\pm$0.74\%$\bullet$    & 64.80$\pm$1.28\%$\bullet$  & 76.81$\pm$0.44\%$\bullet$ \\
                              & FREE      &   88.55$\pm$0.38\%$\bullet$    &  70.32$\pm$0.80\%$\bullet$      &    81.17$\pm$0.36\%$\bullet$     &   32.02$\pm$1.69\%$\bullet$  & 39.22$\pm$0.43\%$\bullet$   & 61.22$\pm$1.24\%$\bullet$ \\
                              & PC        &  92.74$\pm$0.17\%$\bullet$     &  73.18$\pm$0.59\%$\bullet$      &   83.32$\pm$0.28\%$\bullet$      &   43.16$\pm$2.21\%$\bullet$  & 49.53$\pm$1.18\%$\bullet$  & 65.15$\pm$2.05\%$\bullet$  \\
                              & Forward   &   97.67$\pm$0.04\%$\bullet$    &  87.65$\pm$0.24\%$\bullet$      &   88.08$\pm$0.24\%$\bullet$      &   71.92$\pm$1.09\%$\bullet$   &  71.30$\pm$1.16\%$\bullet$   &  77.19$\pm$\%0.76$\bullet$  \\
\midrule
\multirow{5}{*}{Decomposition after Shuffle}  & GA        &   92.08$\pm$0.22\%$\bullet$   &   74.64$\pm$0.67\%$\bullet$   &  79.73$\pm$0.19\%$\bullet$    &  53.12$\pm$0.97\%$\bullet$ &  56.51$\pm$0.89\%$\bullet$  &  63.37$\pm$1.16\%$\bullet$\\
                              & NN        &  92.47$\pm$0.14\%$\bullet$    &  73.88$\pm$0.63\%$\bullet$     &    82.99$\pm$0.13\%$\bullet$     &   36.79$\pm$0.78\%$\bullet$   &  53.78$\pm$0.92\%$\bullet$  & 65.15$\pm$0.73\%$\bullet$  \\
                              & FREE      &   88.99$\pm$0.39\%$\bullet$    &  70.09$\pm$0.74\%$\bullet$  &   81.74$\pm$0.23\%$\bullet$      &   15.16$\pm$2.22\%$\bullet$  &  47.45$\pm$0.98\%$\bullet$  &  50.86$\pm$1.56\%$\bullet$ \\
                              & PC        & 92.94$\pm$0.05\%$\bullet$   &   68.60$\pm$1.32\%$\bullet$     &  82.46$\pm$0.26\%$\bullet$      & 33.16$\pm$0.92\%$\bullet$   &  52.23$\pm$0.88\%$\bullet$ & 64.32$\pm$0.86\%$\bullet$\\
                              & Forward   &  97.49$\pm$0.08\%$\bullet$ &  86.47$\pm$0.39\%$\bullet$   &   87.56$\pm$0.14\%$\bullet$    &   72.16$\pm$0.97\%$\bullet$   &   75.23$\pm$1.02\%   &  79.35$\pm$0.82\%\\
\bottomrule
\end{tabular}
}
}
%\vspace{-0.2cm}
\end{table*}

\textbf{Datasets.\quad} We use five widely-used benchmark datasets MNIST~\cite{lecun1998gradient}, Kuzushiji-MNIST~\cite{clanuwat2018deep}, Fashion-MNIST~\cite{xiao2017fashion}, 20Newsgroups~\cite{lang1995newsweeder}, and CIFAR-10~\cite{krizhevsky2009learning}, and four datasets from the UCI repository~\cite{blake1998uci}. We use four base models including linear model, MLP model ($d$-500-$k$), ResNet (34 layers)~\cite{he2016deep}, and DenseNet (22 layers)~\cite{huang2017densely}.
The detailed descriptions of these datasets with the corresponding base models are provided in Appendix E.1.
To generate MCLs, we instantiate $p(s)=\tbinom{k-1}{s}/(2^k-2)$, $\forall s\in\{1,\cdots,k-1\}$, which means $p(s)$ represents the ratio of the number of label sets whose size is $s$ to the number of all possible label sets.
%Then, we follow the real-world application introduced in Section~\ref{real_world}. Specifically,
For each instance $\boldsymbol{x}$, we first randomly sample $s$ from $p(s)$, and then uniformly and randomly sample a complementary label set $\widebar{Y}$ with size $s$ (i.e., $p(\widebar{Y}) = 1/\tbinom{k-1}{s}$).
%For each dataset, we instantiate $p(s)={|\widebar{\mathcal{Y}}_s|}/{|\widebar{\mathcal{Y}}|}=\tbinom{k-1}{s}/(2^k-2)$, $\forall s\in\{1,\cdots,k-1\}$, and randomly generate MCLs by the following steps. For each training example, we first randomly sample $s$ from $p(s)$. Then, given $s$, we randomly sample a complementary label set $\widebar{Y}$ from $p(\widebar{Y}) = {1}/{|\widebar{\mathcal{Y}}_s|}=1/\tbinom{k-1}{s}$. This above processes clearly accord with the real-world example introduced in Section~\ref{real_world}.

\textbf{Approaches.\quad} We absorb five ordinary complementary-label learning approaches in the two wrappers (introduced in Section~\ref{wrapper}): GA, NN, and Free%\footnote{\url{https://github.com/takashiishida/comp}}
~\cite{Ishida2019Complementary}, PC~\cite{ishida2017learning}, and Forward~\cite{yu2018learning}.%\footnote{\url{https://tongliang-liu.github.io/code.html}}
We also use an unbounded loss CCE and four bounded losses MAE, MSE, GCE~\cite{zhang2018generalized}, and PHuber-CE \cite{menon2020can} in our empirical estimator Eq.~(\ref{emp_risk}). Besides, two upper-bound loss functions LOG and EXP are also inserted into Eq.~(\ref{mae_loss}). In addition, we also compare with a representative partial label learning approach CLPL~\cite{cour2011learning}.
For all the approaches, we adopt the same base model for fair comparison. Learning rate and weight decay are selected from $\{10^{-6},10^{-5},\cdots,10^{-1}\}$. We implement our approach using PyTorch\footnote{\url{www.pytorch.org}}, and use the Adam~\cite{kingma2015adam} optimization method with mini-batch size set to 256 and epoch number set to 250. Hyper-parameters for all the approaches are selected so as to maximize the accuracy on a validation set (10\% of the training set) of complementarily labeled data. All the experiments are conducted on NVIDIA Tesla V100 GPUs.
%(10\% of the training set)
%\subsection{Experimental Results}
%\textbf{Experimental Results of Unbounded and Bounded Loss Functions.}

\textbf{Loss Comparison.\quad} Figure~\ref{fig1} shows the mean and standard deviation of test accuracy of 5 trials, for bounded loss functions MAE, MSE, GCE, PHuber-CE, and unbounded loss function CCE used in our empirical risk estimator Eq.~(\ref{emp_risk}). We also record the mean and standard deviation of training accuracy (the training set is evaluated with ordinary labels) of 5 trials, and put the results in Appendix E.2.
As can be seen from Figure~\ref{fig1}, all the bounded losses are significantly better than the unbounded loss CCE in our formulation. This observation clearly accords with our discussion on the over-fitting issue in Section~\ref{implement}.
In addition, MAE achieves comparable performance compared with other bounded losses in most cases, while it is sometimes inferior to other bounded losses due to its optimization issue~\cite{zhang2018generalized}.
%Due to optimization issue of MAE, we further propose two upper-bound loss functions for the following comparison studies.
%\textbf{Experimental Results of Different Approaches.}
Both the advantage and disadvantage of MAE motivate us to use the upper-bound loss functions EXP and LOG for improving the classification performance.

\textbf{Performance Comparison.\quad}
Table~\ref{uci}, Table~\ref{linear}, and Table~\ref{deep} show the experimental results of different approaches using a linear model or neural networks on the four UCI datasets and the other five benchmark datasets. In table~\ref{deep}, ``CIFAR-10 R" and ``CIFAR-10 D" mean that we use ResNet and DenseNet on CIFAR-10. %Since each approach runs on the same training and test set for each trial, we use
Note that CLPL is a convex approach for partial label learning, which is specially designed with a linear model. Hence CLPL does not appear in Table~\ref{deep}. From the three tables, we can find that equipped with the two wrappers ``Decomposition before Shuffle" and ``Decomposition after Shuffle", ordinary complementary-label learning approaches work well for learning with MCLs. However, they are significantly outperformed by the upper-bound losses in most cases, which also achieve the best performance among all the approaches on various benchmark datasets. %However, the upper-bound losses significantly outperform other other approaches (including the wrappers) in most cases.
%Besides, the two upper-bound losses achieve better performance than MAE in most cases.
%significantly outperform other approaches in most cases.
%Besides, the experimental results on four UCI datasets are provided in Appendix E.3, which show the similar comparison results as Table~\ref{linear} and Table~\ref{deep}.
In addition, we also study the case where the size of each complementary label set $s$ is fixed at $j$ (i.e., $p(s=j)=1$) while increasing $j$ from 1 to $k-2$. The corresponding experimental results are provided in Appendix E.3, which show that the classification accuracy of our approaches increases as $j$ increases. This observation is clearly in accordance with our derived estimation error bound (Theorem~\ref{error_bound}), as the estimation error would decrease if $j$ increases.%, as the supervision information would be richer if more complementary labels are provided.

%the proposed upper-bound loss functions EXP and LOG outperform MAE, especially on CIFAR-10 and 20Newsgroups. This observation demonstrates the superiority of the upper-bound losses over MAE. Besides,
%in most cases, particularly on CIFAR-10%and the improvements are
%equipped with the two wrappers ``Decomposition before Shuffle" and ``Decomposition after Shuffle", ordinary complementary-label learning approaches work well for learning with MCLs. Besides, the upper-bound loss functions EXP and LOG outperform MAE, especially on CIFAR-10 and 20Newsgroups. Such improvements demonstrate the effectiveness of the upper-bound losses.

%As can be seen from Table \ref{linear} and Table \ref{deep},
% the proposed upper-bound loss functions significantly outperform other approaches in most cases. %It is also worth noting that "Decomposition before Shuffle" is generally better than "Decomposition after Shuffle". This is intuitive, since the training process can benefit from the randomness of the used training data.
\section{Conclusion}
In this paper, we propose a novel problem setting called \emph{learning with multiple complementary labels} (MCLs), which is a generation of \emph{complementary-label learning}~\cite{ishida2017learning,yu2018learning,Ishida2019Complementary}. To solve this learning problem, we first design two wrappers that enable us to use arbitrary complementary-label learning approaches for learning with MCLs. However, we find that the supervision information that MCLs hold is conceptually diluted after decomposition. %in the process of the SGD-like algorithm.
Therefore, we further propose an unbiased risk estimator for learning with MCLs, which processes each set of MCLs as a whole. Then, we theoretically derive an estimation error bound, which guarantees the learning consistency.
Although our risk estimator does not rely on specific models or loss functions, we show that bounded loss is generally better than unbounded loss in our empirical risk estimator. In addition, we improve the risk estimator into minimizing properly chosen upper bounds for practical implementation.
%propose two upper-bound surrogate loss functions of the bounded loss Mean Absolute Error (MAE), and demonstrate their effectiveness by both gradient analysis and experimental results.
%It would be more promising to learn from multiple complementarylabels simultaneously,  as the supervision infor-mation would be richer if more complementary labels are provided.
%In addition, we theoretically derive an estimation error bound of our proposed formulation.
Extensive experiments demonstrate the effectiveness of the proposed approaches.
\section*{Acknowledgements}
This research was supported by the National Research Foundation, Singapore under its AI Singapore Programme (AISG Award No: AISG-RP-2019-0013), National Satellite of Excellence in Trustworthy Software Systems (Award No: NSOE-TSS2019-01), and NTU. BH was partially supported by the Early Career Scheme (ECS) through the Research Grants Council of Hong Kong under Grant No.22200720, HKBU Tier-1 Start-up Grant and HKBU CSD Start-up Grant. GN and MS were supported by JST AIP Acceleration Research Grant Number JPMJCR20U3, Japan.
\bibliography{full_version}
\bibliographystyle{icml2020}
\appendix

\onecolumn
\section{Proofs about the Problem Setting}\label{A}
\subsection{Proofs of Theorem 1}\label{A.1}
Firstly, we define the set of all the possible label sets whose size is $j$ as
\begin{gather}
\nonumber
\widebar{\mathcal{Y}}_j := \{Y\mid Y\in\widebar{\mathcal{Y}}, |Y|=j\}.
\end{gather}
Then, by the definition of $\widebar{p}(\boldsymbol{x},\widebar{Y})$, we can obtain
\begin{align}
\nonumber
\int_{\widebar{\mathcal{Y}}}\int_{\mathcal{X}}{\widebar{p}(\boldsymbol{x},\widebar{Y})} \mathrm{d}\boldsymbol{x}\ \mathrm{d}\widebar{Y}&=\int\sum\nolimits_{\widebar{Y}\in\widebar{\mathcal{Y}}}\widebar{p}(\boldsymbol{x},\widebar{Y})\text{d}\boldsymbol{x}\\
\nonumber
&=\int\sum\nolimits_{\widebar{Y}\in\widebar{\mathcal{Y}}}\sum\nolimits_{j=1}^{k-1}\Big(\widebar{p}(\boldsymbol{x},\widebar{Y}\mid s=j)p(s=j)\Big)\text{d}\boldsymbol{x}\\
\nonumber
&=\int\sum\nolimits_{j=1}^{k-1}\sum\nolimits_{\widebar{Y}\in\widebar{\mathcal{Y}}_j}\Big(\widebar{p}(\boldsymbol{x},\widebar{Y}\mid s=j)p(s=j)\Big)\text{d}\boldsymbol{x} \quad\quad (\because \widebar{\mathcal{Y}}_j := \{\widebar{Y}\mid \widebar{Y}\in\mathcal{\widebar{Y}},|\widebar{Y}|=j\})\\
\nonumber
&=\int\sum\nolimits_{j=1}^{k-1}\sum\nolimits_{\widebar{Y}\in\widebar{\mathcal{Y}}_j}\bigg(\frac{1}{\tbinom{k-1}{j}}\sum\nolimits_{y\notin\widebar{Y}}p(\boldsymbol{x},y)p(s=j)\bigg)\text{d}\boldsymbol{x} \quad\quad (\because \text{the definition of } \widebar{p}(\boldsymbol{x},\widebar{Y}\mid s=j))\\
\nonumber
&=\int\sum\nolimits_{j=1}^{k-1}\bigg( \frac{1}{\tbinom{k-1}{j}}\frac{\tbinom{k}{j}(k-j)}{k}\sum\nolimits_{y=1}^kp(\boldsymbol{x},y)p(s=j)\bigg)\text{d}\boldsymbol{x}\quad\quad \big(\because |\widebar{\mathcal{Y}}_j|=\tbinom{k}{j}\big)\\
\nonumber
&=\int\sum\nolimits_{j=1}^{k-1} p(\boldsymbol{x})p(s=j)\text{d}\boldsymbol{x}\\
\nonumber
&=1,
\end{align}
which concludes the proof of Theorem 1.\qed

\subsection{Proof of Lemma 1}\label{A.2}
Let us consider the case where the correct label $y$ is a specific label $i$ $(i\in\{1,2,\cdots,k\})$, then we have
\begin{align}
\nonumber
p(y\in\widebar{Y},y=i\mid \boldsymbol{x},s)=& p(y\in\widebar{Y}\mid y=i,\boldsymbol{x},s)p(y=i\mid \boldsymbol{x},s)\\
\nonumber
=& \sum\nolimits_{C\in\widebar{\mathcal{Y}}}p(y\in\widebar{Y},\widebar{Y}=C\mid y=i,\boldsymbol{x},s)p(y=i\mid \boldsymbol{x},s).
\end{align}
Here, $p(y=i\mid \boldsymbol{x},s) = p(y=i\mid \boldsymbol{x})$ since the labeling rule is independent of $s$. In addition, $\sum\nolimits_{C\in\widebar{\mathcal{Y}}}p(y\in\widebar{Y},\widebar{Y}=C\mid y=i,\boldsymbol{x},s) = \sum\nolimits_{C\in\widebar{\mathcal{Y}}_s}p(y\in\widebar{Y},\widebar{Y}=C\mid y=i,\boldsymbol{x})$ since given the size $s$ of the label set, the whole set of all the possible label sets becomes $\widebar{\mathcal{Y}}_s$.
Then, we can obtain
\begin{align}
\nonumber
p(y\in\widebar{Y},y=i\mid \boldsymbol{x},s)=& \sum\nolimits_{C\in\widebar{\mathcal{Y}}}p(y\in\widebar{Y},\widebar{Y}=C\mid y=i,\boldsymbol{x},s)p(y=i\mid \boldsymbol{x},s)\\
\nonumber
=& \sum\nolimits_{C\in\widebar{\mathcal{Y}}_s}p(y\in\widebar{Y},\widebar{Y}=C\mid y=i,\boldsymbol{x})p(y=i\mid \boldsymbol{x})\\
\nonumber
=& \sum\nolimits_{C\in\widebar{\mathcal{Y}}_s}p(y\in\widebar{Y}\mid\widebar{Y}=C, y=i,\boldsymbol{x})p(y=i\mid\boldsymbol{x})p(\widebar{Y}=C\mid\boldsymbol{x})\\
\nonumber
=& \sum\nolimits_{C\in\widebar{\mathcal{Y}}_s}p(y\in\widebar{Y}|\widebar{Y}=C, y=i,\boldsymbol{x})p(y=i|\boldsymbol{x})p(\widebar{Y}=C),
\end{align}
where the last equality holds due to the fact that for each instance $\boldsymbol{x}$, $\widebar{Y}$ is uniformly and randomly chosen. Since $p(\widebar{Y}=C)=\frac{1}{|{\widebar{\mathcal{Y}}}_s|}$ if $C\in\widebar{\mathcal{Y}}_s$ where $|\widebar{\mathcal{Y}}_s|=\tbinom{k}{s}$, we have
\begin{align}
\nonumber
p(y\in\widebar{Y},y=i\mid \boldsymbol{x},s)=&
\sum\nolimits_{C\in\widebar{\mathcal{Y}}_s}p(y\in\widebar{Y}|\widebar{Y}=C, y=i,\boldsymbol{x})p(y=i|\boldsymbol{x})p(\widebar{Y}=C)\\
\nonumber
=& \frac{1}{\tbinom{k}{s}}\sum\nolimits_{C\in\widebar{\mathcal{Y}}_s}p(y\in\widebar{Y}|\widebar{Y}=C, y=i,\boldsymbol{x})p(y=i|\boldsymbol{x})\\
\nonumber
=& \frac{1}{\tbinom{k}{s}}|\widebar{\mathcal{Y}}_s^i|p(y=i|\boldsymbol{x})\quad\quad \big(\because\widebar{\mathcal{Y}}_s^i := \{\widebar{Y}\in\widebar{\mathcal{Y}}_s\ |\ i\in\widebar{Y}\}\big)\\
\nonumber
=&\frac{\tbinom{k-1}{s-1}}{\tbinom{k}{s}}p(y=i|\boldsymbol{x})\quad\quad \Big(\because|\widebar{\mathcal{Y}}_s^i| = \tbinom{k-1}{s-1}\Big)\\
\nonumber
=& \frac{s}{k}p(y=i|\boldsymbol{x}).
\end{align}
By further summing up the both side over all the possible $i$, we can obtain
\begin{align}
\nonumber
p(y\in\widebar{Y}|\boldsymbol{x},s) = \frac{s}{k},
\end{align}
which concludes the proof of Lemma 1.\qed
\subsection{Proof of Theorem 2}\label{A.3}
Let us express $p(\widebar{Y}|y\notin\widebar{Y},\boldsymbol{x},s)$ as
\begin{align}
\nonumber
p(\widebar{Y}|y\notin\widebar{Y},\boldsymbol{x},s)=&\frac{p(y\notin\widebar{Y},\widebar{Y}|\boldsymbol{x},s)}{p(y\notin\widebar{Y}|\boldsymbol{x},s)}\\
\nonumber
=&\frac{p(y\notin\widebar{Y}|\widebar{Y},\boldsymbol{x},s)p(\widebar{Y}|\boldsymbol{x},s)}{p(y\notin\widebar{Y}|\boldsymbol{x},s)}\\
\nonumber
=&\frac{p(y\notin\widebar{Y}|\widebar{Y},\boldsymbol{x},s)p(\widebar{Y}|s)}{p(y\notin\widebar{Y}|\boldsymbol{x},s)},
\end{align}
where the last equality holds because $\widebar{Y}$ is influenced by the size $s$, and for each instance $\boldsymbol{x}$, $\widebar{Y}$ is uniformly and randomly chosen. Note that given $s$, there are $|\widebar{\mathcal{Y}}_s|$ possible label sets, thus $p(\widebar{Y}|s)=\frac{1}{|\widebar{\mathcal{Y}}_s|}$ where $|\widebar{\mathcal{Y}}_s|=\tbinom{k}{s}$. In this way, we have
\begin{align}
\nonumber
p(\widebar{Y}|y\notin\widebar{Y},\boldsymbol{x},s)=& \frac{p(y\notin\widebar{Y}|\widebar{Y},\boldsymbol{x},s)p(\widebar{Y}|s)}{p(y\notin\widebar{Y}|\boldsymbol{x},s)}
\\
\nonumber
=&\frac{1}{\tbinom{k}{s}}\frac{p(y\notin\widebar{Y}|\widebar{Y},\boldsymbol{x},s)}{1-p(y\in\widebar{Y}|\boldsymbol{x},s)}\\
\nonumber
=&\frac{1}{\tbinom{k}{s}}\frac{1}{1-\frac{s}{k}}p(y\notin\widebar{Y}|\widebar{Y},\boldsymbol{x},s)\quad\quad \Big(\because\text{by Lemma 1, }p(y\in\widebar{Y}|\boldsymbol{x},s)=\frac{s}{k}\Big)\\
\nonumber
=&\frac{1}{\tbinom{k}{s}}\frac{k}{k-s}\sum\nolimits_{y\notin\widebar{Y}}p(y|\boldsymbol{x},s)\\
\nonumber
=&\frac{1}{\tbinom{k-1}{s}}\sum\nolimits_{y\notin\widebar{Y}}p(y|\boldsymbol{x}).
\end{align}
By multiplying $p(\boldsymbol{x})$ on both side, we have
\begin{gather}
\nonumber
p(\boldsymbol{x},\widebar{Y}|y\notin\widebar{Y},s)=\frac{1}{\tbinom{k-1}{s}}\sum\nolimits_{y\notin\widebar{Y}}p(\boldsymbol{x},y).
\end{gather}
Then taking into account the variable $s$, we have
\begin{align}
\nonumber
p(\boldsymbol{x},\widebar{Y}|y\notin\widebar{Y})=& \sum\nolimits_{j=1}^{k-1}p(s=j)p(\boldsymbol{x},\widebar{Y}|y\notin\widebar{Y},s=j)\\
\nonumber
=& \sum\nolimits_{j=1}^{k-1}p(s=j)\frac{1}{\tbinom{k-1}{j}}\sum\nolimits_{y\notin\widebar{Y}}p(\boldsymbol{x},y)\\
\nonumber
=& \widebar{p}(\boldsymbol{x},\widebar{Y}),
\end{align}
which concludes the proof.\qed

\section{Proofs of the Unbiased Risk Estimator}
\subsection{Proof of Lemma 2}~\label{B.1}
According to our defined distribution, we have already obtained
\begin{gather}
\nonumber
\widebar{p}(\boldsymbol{x},\widebar{Y} \mid s = j) = \frac{1}{\tbinom{k-1}{j}}\sum\nolimits_{y^\prime\notin\widebar{Y}}p(\boldsymbol{x},y^\prime).
\end{gather}
Then, we can obtain the following equality by operating $\sum\nolimits_{\widebar{Y}\in\widebar{\mathcal{Y}}_j^y}$ on both the left and the right hand side:
\begin{gather}
\label{eq_theo2}
\sum\nolimits_{\widebar{Y}\in\widebar{\mathcal{Y}}_j^y}\widebar{p}(\widebar{Y} \mid \boldsymbol{x},s = j) = \frac{1}{\tbinom{k-1}{j}}\sum\nolimits_{\widebar{Y}\in\widebar{\mathcal{Y}}_j^y}\sum\nolimits_{y^\prime\notin\widebar{Y}}p(y^\prime\mid\boldsymbol{x}),
\end{gather}
where $\widebar{\mathcal{Y}}_j^y := \{\widebar{Y}\in\widebar{\mathcal{Y}}\mid y\in\widebar{Y},|\widebar{Y}|=j\}$.
In this way, the right hand side of the above equality can be transformed by the following derivations:
\begin{align}
\nonumber
\frac{1}{\tbinom{k-1}{j}}\sum\nolimits_{\widebar{Y}\in\widebar{\mathcal{Y}}_j^y}\sum\nolimits_{y^\prime\notin\widebar{Y}}p(y^\prime\mid\boldsymbol{x})
=& \frac{1}{\tbinom{k-1}{j}}\sum\nolimits_{\widebar{Y}\in\widebar{\mathcal{Y}}_j^y}\bigg(1-\sum\nolimits_{y^\prime\in\widebar{Y}}p(y^\prime\mid\boldsymbol{x})\bigg)\\
\nonumber
=&\frac{|\widebar{\mathcal{Y}}_j^y|}{\tbinom{k-1}{j}} - \frac{1}{\tbinom{k-1}{j}}\sum\nolimits_{\widebar{Y}\in\widebar{\mathcal{Y}}_j^y}\sum\nolimits_{y^\prime\in\widebar{Y}}p(y^\prime\mid\boldsymbol{x})\\
\nonumber
=&\frac{\tbinom{k-1}{j-1}}{\tbinom{k-1}{j}} - \frac{1}{\tbinom{k-1}{j}}\sum\nolimits_{y^\prime}\sum\nolimits_{\widebar{Y}^\prime\in\{\widebar{Y}^\prime\in\widebar{\mathcal{Y}}_j^y|y^\prime\in\widebar{Y}^\prime\}}p(y^\prime\mid\boldsymbol{x})\quad\quad \big(\because |\widebar{\mathcal{Y}}_j^y| = \tbinom{k-1}{j-1}\big)\\
\nonumber
=&\frac{j}{k-j}-\frac{1}{\tbinom{k-1}{j}}\left\{\tbinom{k-1}{j-1}p(y\mid\boldsymbol{x}) + \tbinom{k-2}{j-2}\sum\nolimits_{y^\prime\neq y}p(y^\prime\mid\boldsymbol{x})\right\}\\
\nonumber
=&\frac{j}{k-j}-\frac{1}{\tbinom{k-1}{j}}\left\{\tbinom{k-1}{j-1}p(y\mid\boldsymbol{x}) + \tbinom{k-2}{j-2}\big(1-p(y\mid\boldsymbol{x})\big)\right\}\\
\nonumber
=&\frac{j}{k-j}-\frac{1}{\tbinom{k-1}{j}}\left\{\tbinom{k-2}{j-2} + \tbinom{k-2}{j-1}p(y\mid\boldsymbol{x})\right\} \quad\quad \big(\because \tbinom{k-2}{j-1} = \tbinom{k-1}{j-1} - \tbinom{k-2}{j-2}\big)\\
\nonumber
=&\frac{j}{k-j}-\frac{j(j-1)}{(k-j)(k-1)}-\frac{j}{k-1}p(y\mid\boldsymbol{x})\\
\label{right_hand}
=&\frac{j}{k-1}-\frac{j}{k-1}p(y\mid\boldsymbol{x}).
\end{align}
Combing Eq.~(\ref{eq_theo2}) and Eq.~(\ref{right_hand}), we obtain
\begin{align}
\label{proof_b1}
p(y\mid \boldsymbol{x}, s = j) = p(y\mid\boldsymbol{x})= 1-\frac{k-1}{j}\sum\nolimits_{\widebar{Y}\in\widebar{\mathcal{Y}}_j^y}\widebar{p}(\widebar{Y} \mid \boldsymbol{x}, s = j).
\end{align}
In the end, by taking into account the variable $s$, we have
\begin{align}
\nonumber
p(y\mid\boldsymbol{x})
& = \sum\nolimits_{j=1}^{k-1} p(s=j)p(y \mid \boldsymbol{x}, s=j) \\
\nonumber
&= \sum\nolimits_{j=1}^{k-1} p(s=j)\Big(1-\frac{k-1}{j}\sum\nolimits_{\widebar{Y}\in\widebar{\mathcal{Y}}_j^y}\widebar{p}(\widebar{Y}\mid\boldsymbol{x},s=j)\Big) \\
\nonumber
&= 1 - \sum\nolimits_{j=1}^{k-1} \Big(\frac{k-1}{j}\sum\nolimits_{\widebar{Y}\in\widebar{\mathcal{Y}}_j^y}\widebar{p}(\widebar{Y}, s=j\mid\boldsymbol{x})\Big),
\end{align}
which concludes the proof of Lemma 2.
\subsection{Proof of Theorem 3}\label{B.2}
It is intuitive to obtain
\begin{align}
\nonumber
R(f) = \mathbb{E}_{p(\boldsymbol{x},y)}\big[\mathcal{L}\big(f(\boldsymbol{x}),y\big)\big] = \sum\nolimits_{j=1}^{k-1} p(s=j)\mathbb{E}_{p(\boldsymbol{x},y \mid s=j)}\big[\mathcal{L}\big(f(\boldsymbol{x}),y\big)\big].
\end{align}
Then, we express the right hand side for each $j \in \{1, \ldots, k-1\}$ as
\begin{align}
\nonumber
\mathbb{E}_{p(\boldsymbol{x},y \mid s=j)}\big[\mathcal{L}\big(f(\boldsymbol{x}),y\big)\big]
\nonumber
&=\mathbb{E}_{p(\boldsymbol{x} \mid s=j)}\mathbb{E}_{p(y|\boldsymbol{x}, s=j)}\big[\mathcal{L}\big(f(\boldsymbol{x}),y\big)\big]\\
\nonumber
&=\mathbb{E}_{p(\boldsymbol{x} \mid s=j)}\bigg[\sum\nolimits_{y=1}^kp(y|\boldsymbol{x},s=j)\mathcal{L}\big(f(\boldsymbol{x}),y\big)\bigg]\quad\quad \\
\nonumber
&=\mathbb{E}_{p(\boldsymbol{x} \mid s=j)}\bigg[\sum\nolimits_{y=1}^k\bigg(1-\frac{k-1}{j}\sum\nolimits_{\widebar{Y}\in\widebar{\mathcal{Y}}_j^y}\widebar{p}(\widebar{Y}|\boldsymbol{x}, s=j)\bigg)\mathcal{L}\big(f(\boldsymbol{x}),y\big)\bigg]\quad\quad (\because \text{Eq. }(\ref{proof_b1}))\\
\nonumber
&=\mathbb{E}_{p(\boldsymbol{x} \mid s=j)}\bigg[\sum\nolimits_{y=1}^k\mathcal{L}\big(f(\boldsymbol{x}),y\big) - \frac{k-1}{j}\sum\nolimits_{y=1}^k\sum\nolimits_{\widebar{Y}\in\widebar{\mathcal{Y}}_j^y}\widebar{p}(\widebar{Y}|\boldsymbol{x}, s=j)\mathcal{L}\big(f(\boldsymbol{x}),y\big)\bigg]\\
\nonumber
&=\mathbb{E}_{p(\boldsymbol{x} \mid s=j)}\bigg[\sum\nolimits_{y=1}^k\mathcal{L}\big(f(\boldsymbol{x}),y\big) - \frac{k-1}{j}\sum\nolimits_{\widebar{Y}\in\widebar{\mathcal{Y}}_j}\sum\nolimits_{y^\prime\in\widebar{Y}}\widebar{p}(\widebar{Y}|\boldsymbol{x}, s=j)\mathcal{L}\big(f(\boldsymbol{x}),y\big)\bigg]\\
\nonumber
&=\mathbb{E}_{p(\boldsymbol{x} \mid s=j)}\bigg[\sum\nolimits_{y=1}^k\mathcal{L}\big(f(\boldsymbol{x}),y\big) - \frac{k-1}{j}\sum\nolimits_{\widebar{Y}\in\widebar{\mathcal{Y}}_j}\widebar{p}(\widebar{Y}|\boldsymbol{x}, s=j)\Big(\sum\nolimits_{y^\prime\in\widebar{Y}}\mathcal{L}\big(f(\boldsymbol{x}),y^\prime\big)\Big)\bigg]\\
\nonumber
&=\mathbb{E}_{p(\boldsymbol{x} \mid s=j)}\mathbb{E}_{\widebar{p}(\widebar{Y}|\boldsymbol{x}, s=j)}\bigg[\sum\nolimits_{y=1}^k\mathcal{L}\big(f(\boldsymbol{x}),y\big) - \frac{k-1}{j}\sum\nolimits_{y^\prime\in\widebar{Y}}\mathcal{L}\big(f(\boldsymbol{x}),y^\prime\big)\bigg]\\
\nonumber
&=\mathbb{E}_{\widebar{p}(\boldsymbol{x},\widebar{Y} \mid s=j)}\bigg[\sum\nolimits_{y\notin\widebar{Y}}\mathcal{L}\big(f(\boldsymbol{x}),y^\prime\big) - \frac{k-1-j}{j}\sum\nolimits_{y^\prime\in\widebar{Y}}\mathcal{L}\big(f(\boldsymbol{x}),y^\prime\big)\bigg]\\
\nonumber
&=\mathbb{E}_{\widebar{p}(\boldsymbol{x},\widebar{Y} \mid s=j)}[\widebar{\mathcal{L}}_j\big(f(\boldsymbol{x}),\widebar{Y}\big)]\\
\nonumber
&=\widebar{R}_j(f).
\end{align}
In this way, we can obtain $R(f)=\sum\nolimits_{j=1}^{k-1}p(s=j)\widebar{R}_{j}(f)$, which concludes the proof of Theorem 3.\qed

\section{Proof of Theorem 4}\label{C}%Theorem~\ref{error_bound} and Corollary~\ref{unfix_error_bound}
Recall that the expected risk and empirical risk are represented as
\begin{align}
\nonumber
R(f) &= \sum\nolimits_{j=1}^{k-1}p(s=j)\widebar{R}_j(f) =
\sum\nolimits_{j=1}^{k-1}p(s=j)\mathbb{E}_{p(\boldsymbol{x},\widebar{Y}\mid s=j)}[\widebar{\mathcal{L}}_j\big(f(\boldsymbol{x}),\widebar{Y}\big)],\\
\nonumber
\widehat{R}(f) &= \sum\nolimits_{j=1}^{k-1}\frac{p(s=j)}{n_j}\sum\nolimits_{i=1}^{n_j}\widebar{\mathcal{L}}_j\big(f(\boldsymbol{x}_i),\widebar{Y}_i\big).
\end{align}
Here, with a slight abuse of notation, we simply write $\widebar{R}_j(f)$ as $R_j(f)$, and define $\widehat{R}_j(f) = 1/{n_j}\sum\nolimits_{i=1}^{n_j}\widebar{\mathcal{L}}_j\big(f(\boldsymbol{x}_i),\widebar{Y}_i\big)$. Thus we have $R(f)=\sum_{j=1}^{k-1}p(s=j){R}_j(f)$ and $\widehat{R}(f)=\sum_{j=1}^{k-1}p(s=j)\widehat{R}_j(f)$. Since $f^\star=\arg\min\nolimits_{f\in\mathcal{F}}R(f)$ and $\widehat{f}=\arg\min\nolimits_{f\in\mathcal{F}}\widehat{R}(f)$, we can obtain the following lemma.
\setcounter{lemma}{2}
\begin{lemma}
\label{basic_lemma}
The following inequality holds:
\begin{gather}
\nonumber
R(\widehat{f}) - R(f^\star) \leq  2\sum\nolimits_{j=1}^{k-1}p(s=j)\sup_{f\in\mathcal{F}}\left|\widehat{R}_j(f) - R_j(f)\right|.
\end{gather}
\end{lemma}
\begin{proof}
It would be intuitive to obtain
\begin{align}
\nonumber
R(\widehat{f}) - R(f^\star) =& R(\widehat{f}) - \widehat{R}(\widehat{f}) + \widehat{R}(\widehat{f}) - R(f^\star) \\
\nonumber
\leq& R(\widehat{f}) - \widehat{R}(\widehat{f}) + R(\widehat{f})  - R(f^\star)\\
\nonumber
\leq&2\sup_{f\in\mathcal{F}}\left|\widehat{R}(f) - R(f)\right|\\
\nonumber
=& 2\sup_{f\in\mathcal{F}}\left|\sum\nolimits_{j=1}^{k-1}p(s=j)\widehat{R}_j(f) - \sum\nolimits_{j=1}^{k-1}p(s=j)R_j(f)\right|\\
\nonumber
\leq& 2\sum\nolimits_{j=1}^{k-1}p(s=j)\sup_{f\in\mathcal{F}}\left|\widehat{R}_j(f) - R_j(f)\right|,
\end{align}
which concludes the proof of Lemma \ref{basic_lemma}.
\end{proof}
In this way, we will bound $\sup\nolimits_{f\in\mathcal{F}}\left|\widehat{R}_j(f) - R_j(f)\right|$ for $j=\{1,\ldots,k-1\}$. Before that, we define a function space as
\begin{gather}
\nonumber
\mathcal{H}_j := \{(\boldsymbol{x},\widebar{Y})\in\mathcal{X}\times\widebar{\mathcal{Y}}_j\mapsto \widebar{\mathcal{L}}_j\big(f(\boldsymbol{x}),\widebar{Y}\big)\ |\ f\in\mathcal{F}\},
\end{gather}
where
\begin{gather}
\nonumber
\widebar{\mathcal{L}}_j\big(f(\boldsymbol{x}),\widebar{Y}\big) := \sum\nolimits_{y\notin\widebar{Y}}\mathcal{L}\big(f(\boldsymbol{x}),y\big)-\frac{k-1-j}{j}\sum\nolimits_{y^\prime\in\widebar{Y}}\mathcal{L}\big(f(\boldsymbol{x}),y^\prime\big).
\end{gather}
Besides, we introduce the definition of \emph{Rademacher complexity}~\cite{bartlett2002rademacher}.
\begin{definition}[Rademacher complexity~\cite{bartlett2002rademacher}]
Let $Z_1,\dots,Z_n$ be $n$ i.i.d. random variables drawn from a probability distribution $\mathcal{D}$, $\mathcal{H}=\{h:\mathcal{Z}\rightarrow\mathbb{R}\}$ be a class of measurable functions. Then the expected Rademacher complexity of $\mathcal{H}$ is defined as
\begin{gather}
\nonumber
\mathfrak{R}_n(\mathcal{H})=\mathbb{E}_{Z_1,\dots,Z_n\sim\mathcal{D}}\mathbb{E}_{\boldsymbol{\sigma}}\bigg[\sup\nolimits_{h\in\mathcal{H}}\frac{1}{n}\sum\nolimits_{i=1}^n\sigma_ih(Z_i)\bigg],
\end{gather}
where $\boldsymbol{\sigma}=(\sigma_1,\dots,\sigma_n)$ are Rademacher variables taking the value from $\{-1,+1\}$ with even probabilities.
\end{definition}
Then, we have the following lemma.
\begin{lemma}
\label{temp_lemma}
Let $C_{\mathcal{L}}=\sup\nolimits_{\boldsymbol{x}\in\mathcal{X},f\in\mathcal{F},y\in\mathcal{Y}}\mathcal{L}\big(f(\boldsymbol{x}),y\big)$. Then, for all $j=\{1,\ldots,k-1\}$, for any $\delta>0$, with probability at least $1-\delta$,
\begin{gather}
\sup_{f\in\mathcal{F}}\left|\widehat{R}_j(f) - R_j(f)\right|\leq 2\widebar{\mathfrak{R}}_{n_j}(\mathcal{H}_j)+(2k-2j-1)C_{\mathcal{L}}\sqrt{\frac{\log\frac{2}{\delta}}{2n_j}},
\end{gather}
where
\begin{gather}
\widebar{\mathfrak{R}}_{n_j}(\mathcal{H}_j) = \mathbb{E}_{(\boldsymbol{x}_i,\widebar{Y}_i)\sim\widebar{p}(\boldsymbol{x},\widebar{Y}\mid s=j)}\mathbb{E}_{\boldsymbol{\sigma}}\bigg[\sup_{h\in\mathcal{H}_j}\frac{1}{n_j}\sum\nolimits_{i=1}^{n_j}\sigma_ih(\boldsymbol{x}_i,\widebar{Y}_i)\bigg].
\end{gather}
\end{lemma}
\begin{proof}
To prove this lemma, we first show that the single direction $\sup_{f\in\mathcal{F}}\big(\widehat{R}_j(f)-R_j(f)\big)$ is bounded with probability at least $1-\frac{\delta}{2}$, and the other direction can be similarly proved. By the definition of $\widebar{\mathcal{L}}_j$, we can easily know the possible maximum of $\widebar{\mathcal{L}}_j$ is $(k-j)C_{\mathcal{L}}$, and the possible minimum is $-(k-1-j)C_{\mathcal{L}}$. Suppose an example $(\boldsymbol{x}_i,\widebar{Y}_i)$ is replaced by another arbitrary example $(\boldsymbol{x}_i^\prime,\widebar{Y}_i^\prime)$, then the change of $\sup_{f\in\mathcal{F}}\big(\widehat{R}_j(f)-R_j(f)\big)$ is no greater than $((2k-2j-1)C_{\mathcal{L}})/n_j$. Then, by applying \emph{McDiarmid's inequality}~\cite{McDiarmid1989SurveysIC}, for any $\delta>0$, with probability at least $1-\frac{\delta}{2}$,
\begin{gather}
\label{first_side}
\sup_{f\in\mathcal{F}}\big(\widehat{R}_j(f)-R_j(f)\big)\leq\mathbb{E}\Big[\sup_{f\in\mathcal{F}}\big(\widehat{R}_j(f)-R_j(f)\big)\Big]+(2k-2j-1)C_{\mathcal{L}}\sqrt{\frac{\log\frac{2}{\delta}}{2n_j}}.
\end{gather}
In addition, it is routine~\cite{mohri2012foundations} to show
\begin{gather}
\label{rade_h}
\mathbb{E}\Big[\sup_{f\in\mathcal{F}}\big(\widehat{R}_j(f)-R_j(f)\big)\Big]\leq 2\widebar{\mathfrak{R}}_{n_j}(\mathcal{H}_j),
\end{gather}
Combing Eq.~(\ref{first_side}) and Eq.~(\ref{rade_h}), we have for any $\delta>0$, with probability at least $1-\frac{\delta}{2}$,
\begin{gather}
\sup_{f\in\mathcal{F}}\big(\widehat{R}_j(f)-R_j(f)\big)\leq 2\widebar{\mathfrak{R}}_{n_j}(\mathcal{H}_j)+(2k-2j-1)C_{\mathcal{L}}\sqrt{\frac{\log\frac{2}{\delta}}{2n_j}}.
\end{gather}
By further taking into account the other side $\sup_{f\in\mathcal{F}}\big(R_j(f)-\widehat{R}_j(f)\big)$, we have for any $\delta>0$, with probability at least $1-\delta$,
\begin{gather}
\nonumber
\sup_{f\in\mathcal{F}}\left|\widehat{R}_j(f)-R_j(f)\right| \leq 2\widebar{\mathfrak{R}}_{n_j}(\mathcal{H}_j)+(2k-2j-1)C_{\mathcal{L}}\sqrt{\frac{\log\frac{2}{\delta}}{2n_j}}.
\end{gather}
which concludes the proof of Lemma~\ref{temp_lemma}.
\end{proof}
Next, we will bound the expected Rademacher complexity of the function space $\mathcal{H}_j$, i.e., $\widebar{\mathfrak{R}}_{n_j}(\mathcal{H}_j)$. %Recall that we have already defined $\mathcal{G}_y = \{g : \boldsymbol{x}\mapsto f_y(\boldsymbol{x})\mid f\in\mathcal{F}\}$, then we have the following lemma.
\begin{lemma}
\label{rade_func}
Assume the loss function $\mathcal{L}\big(f(\boldsymbol{x}),y\big)$ is $\rho$-Lipschitz with respect to $f(\boldsymbol{x})$ $(0<\rho<\infty)$ for all $y\in\mathcal{Y}$. Then, for all $j=\{1,\ldots,k-1\}$, the following inequality holds:
\begin{gather}
\nonumber
\widebar{\mathfrak{R}}_{n_j}(\mathcal{H}_j)\leq\frac{\rho(k-1)}{j}\sum\nolimits_{y=1}^k\mathfrak{R}_{n_j}(\mathcal{G}_y),
\end{gather}
where
\begin{align}
\nonumber
\mathcal{G}_y =& \{g : \boldsymbol{x}\mapsto f_y(\boldsymbol{x})\mid f\in\mathcal{F}\},\\
\nonumber
\mathfrak{R}_{n_j}(\mathcal{G}_y) =& \mathbb{E}_{\boldsymbol{x}_i\sim p(\boldsymbol{x})}\mathbb{E}_{\boldsymbol{\sigma}}\bigg[\sup_{g\in\mathcal{G}_y}\frac{1}{n_j}\sum\nolimits_{i=1}^{n_j}g(\boldsymbol{x}_i)\bigg].
\end{align}
\end{lemma}
\begin{proof}
The expected Rademacher complexity of $\mathcal{H}_j$ can be expressed as
\begin{align}
\nonumber
\widebar{\mathfrak{R}}_{n_j}(\mathcal{H}_j) =& \mathbb{E}_{(\boldsymbol{x}_i,\widebar{Y}_i)\sim\widebar{p}(\boldsymbol{x},\widebar{Y}\mid s=j)}\mathbb{E}_{\boldsymbol{\sigma}}\bigg[\sup_{h\in\mathcal{H}_j}\frac{1}{n_j}\sum\nolimits_{i=1}^{n_j}\sigma_ih(\boldsymbol{x}_i,\widebar{Y}_i)\bigg]\\
\nonumber
=&\mathbb{E}_{(\boldsymbol{x}_i,\widebar{Y}_i)\sim\widebar{p}(\boldsymbol{x},\widebar{Y}\mid s=j)}\mathbb{E}_{\boldsymbol{\sigma}}\Bigg[\sup_{f\in\mathcal{F}}\frac{1}{n_j}\sum\nolimits_{i=1}^{n_j}\sigma_i\bigg(\sum\nolimits_{y\notin\widebar{Y}_i}\mathcal{L}\big(f(\boldsymbol{x}),y\big) - \frac{k-j-1}{j}\sum\nolimits_{y^\prime\in\widebar{Y}_i}\mathcal{L}\big(f(\boldsymbol{x}),y^\prime\big)\bigg)\Bigg]\\
\nonumber
\leq& \mathbb{E}_{(\boldsymbol{x}_i,\widebar{Y}_i)\sim\widebar{p}(\boldsymbol{x},\widebar{Y}\mid s=j)}\mathbb{E}_{\boldsymbol{\sigma}}\Bigg[\sup_{f\in\mathcal{F}}\frac{1}{n_j}\sum\nolimits_{i=1}^{n_j}\sigma_i\bigg(\sum\nolimits_{y\notin\widebar{Y}_i}\mathcal{L}\big(f(\boldsymbol{x}),y\big)\bigg)\Bigg] \\
\nonumber
&\quad\quad\quad\quad\quad\quad\quad\quad\quad\quad\quad+\mathbb{E}_{(\boldsymbol{x}_i,\widebar{Y}_i)\sim\widebar{p}(\boldsymbol{x},\widebar{Y}\mid s=j)}\mathbb{E}_{\boldsymbol{\sigma}}\Bigg[\sup_{f\in\mathcal{F}}\frac{1}{n_j}\sum\nolimits_{i=1}^{n_j}\sigma_i\bigg(\frac{k-j-1}{j}\sum\nolimits_{y^\prime\in\widebar{Y}_i}\mathcal{L}\big(f(\boldsymbol{x}),y^\prime\big)\bigg)\Bigg].
\end{align}
Here, we introduce random variables $\alpha_{i,y}=\mathbb{I}[y\in\widebar{Y}_i],\ \forall i\in\{1,\cdots,n\},y\in\mathcal{Y}$, where $\mathbb{I}[\cdot]$ denotes the indicator function. In other words, given a complementary label set $\widebar{Y}_i$, if a specific label $y$ satisfies the condition $y\in\widebar{Y}_i$, then $\mathbb{I}[y\in\widebar{Y}_i]=1$, otherwise $\mathbb{I}[y\in\widebar{Y}_i]=0$.
Then, we can obtain
\begin{align}
\nonumber
\widebar{\mathfrak{R}}_{n_j}(\mathcal{H}_j)\leq& \mathbb{E}_{(\boldsymbol{x}_i,\widebar{Y}_i)\sim\widebar{p}(\boldsymbol{x},\widebar{Y}\mid s=j)}\mathbb{E}_{\boldsymbol{\sigma}}\Bigg[\sup_{f\in\mathcal{F}}\frac{1}{n_j}\sum\nolimits_{i=1}^{n_j}\sigma_i\bigg(\sum\nolimits_{y\notin\widebar{Y}_i}\mathcal{L}\big(f(\boldsymbol{x}),y\big)\bigg)\Bigg] \\
\nonumber
&\quad\quad\quad\quad\quad\quad\quad\quad\quad\quad\quad+\mathbb{E}_{(\boldsymbol{x}_i,\widebar{Y}_i)\sim\widebar{p}(\boldsymbol{x},\widebar{Y}\mid s=j)}\mathbb{E}_{\boldsymbol{\sigma}}\Bigg[\sup_{f\in\mathcal{F}}\frac{1}{n_j}\sum\nolimits_{i=1}^{n_j}\sigma_i\bigg(\frac{k-j-1}{j}\sum\nolimits_{y^\prime\in\widebar{Y}_i}\mathcal{L}\big(f(\boldsymbol{x}),y^\prime\big)\bigg)\Bigg]\\
\nonumber
=&\mathbb{E}_{(\boldsymbol{x}_i,\widebar{Y}_i)\sim\widebar{p}(\boldsymbol{x},\widebar{Y}\mid s=j)}\mathbb{E}_{\boldsymbol{\sigma}}\Bigg[\sup_{f\in\mathcal{F}}\frac{1}{n_j}\sum\nolimits_{i=1}^{n_j}\sigma_i\bigg(\sum\nolimits_{y=1}^k(1-\alpha_{i,y})\mathcal{L}\big(f(\boldsymbol{x}),y\big)\bigg)\Bigg] \\
\nonumber
&\quad\quad\quad\quad\quad\quad\quad\quad\quad\quad+\mathbb{E}_{(\boldsymbol{x}_i,\widebar{Y}_i)\sim\widebar{p}(\boldsymbol{x},\widebar{Y}\mid s=j)}\mathbb{E}_{\boldsymbol{\sigma}}\Bigg[\sup_{f\in\mathcal{F}}\frac{1}{n_j}\sum\nolimits_{i=1}^{n_j}\sigma_i\bigg(\frac{k-j-1}{j}\sum\nolimits_{y=1}^k\alpha_{i,y}\mathcal{L}\big(f(\boldsymbol{x}),y\big)\bigg)\Bigg]\\
\nonumber
=&\mathbb{E}_{(\boldsymbol{x}_i,\widebar{Y}_i)\sim\widebar{p}(\boldsymbol{x},\widebar{Y}\mid s=j)}\mathbb{E}_{\boldsymbol{\sigma}}\Bigg[\sup_{f\in\mathcal{F}}\frac{1}{n_j}\sum\nolimits_{i=1}^{n_j}\sigma_i\bigg(\sum\nolimits_{y=1}^k\frac{1}{2}(1-2\alpha_{i,y}+1)\mathcal{L}\big(f(\boldsymbol{x}),y\big)\bigg)\Bigg] \\
\nonumber
&\quad\quad\quad\quad\quad+\mathbb{E}_{(\boldsymbol{x}_i,\widebar{Y}_i)\sim\widebar{p}(\boldsymbol{x},\widebar{Y}\mid s=j)}\mathbb{E}_{\boldsymbol{\sigma}}\Bigg[\sup_{f\in\mathcal{F}}\frac{1}{n_j}\sum\nolimits_{i=1}^{n_j}\sigma_i\bigg(\frac{k-j-1}{j}\sum\nolimits_{y=1}^k\frac{1}{2}(2\alpha_{i,y}-1+1)\mathcal{L}\big(f(\boldsymbol{x}),y\big)\bigg)\Bigg]\\
\nonumber
=&\mathbb{E}_{(\boldsymbol{x}_i,\widebar{Y}_i)\sim\widebar{p}(\boldsymbol{x},\widebar{Y}\mid s=j)}\mathbb{E}_{\boldsymbol{\sigma}}\Bigg[\sup_{f\in\mathcal{F}}\frac{1}{2n_j}\sum\nolimits_{i=1}^{n_j}\bigg(\sum\nolimits_{y=1}^k(1-2\alpha_{i,y})\sigma_i\mathcal{L}\big(f(\boldsymbol{x}),y\big)+\sigma_i\mathcal{L}\big(f(\boldsymbol{x}),y\big)\bigg)\Bigg] \\
\nonumber
+&\mathbb{E}_{(\boldsymbol{x}_i,\widebar{Y}_i)\sim\widebar{p}(\boldsymbol{x},\widebar{Y}\mid s=j)}\mathbb{E}_{\boldsymbol{\sigma}}\Bigg[\sup_{f\in\mathcal{F}}\frac{1}{2n_j}\sum\nolimits_{i=1}^{n_j}\bigg(\frac{k-j-1}{j}\sum\nolimits_{y=1}^k(2\alpha_{i,y}-1)\sigma_i\mathcal{L}\big(f(\boldsymbol{x}),y^\prime\big)+\sigma_i\mathcal{L}\big(f(\boldsymbol{x}),y^\prime\big)\bigg)\Bigg].
\end{align}
Here, because $(1-2\alpha_{i,y})\sigma_i$ and $(2\alpha_{i,y}-1)\sigma_i$, and $\sigma_i$ follow the same distribution, we have
\begin{align}
\nonumber
\widebar{\mathfrak{R}}_{n_j}(\mathcal{H}_j) \leq &\mathbb{E}_{(\boldsymbol{x}_i,\widebar{Y}_i)\sim\widebar{p}(\boldsymbol{x},\widebar{Y}\mid s=j)}\mathbb{E}_{\boldsymbol{\sigma}}\Bigg[\sup_{f\in\mathcal{F}}\frac{1}{2n_j}\sum\nolimits_{i=1}^{n_j}\bigg(\sum\nolimits_{y=1}^k(1-2\alpha_{i,y})\sigma_i\mathcal{L}\big(f(\boldsymbol{x}),y\big)+\sigma_i\mathcal{L}\big(f(\boldsymbol{x}),y\big)\bigg)\Bigg] \\
\nonumber
+&\mathbb{E}_{(\boldsymbol{x}_i,\widebar{Y}_i)\sim\widebar{p}(\boldsymbol{x},\widebar{Y}\mid s=j)}\mathbb{E}_{\boldsymbol{\sigma}}\Bigg[\sup_{f\in\mathcal{F}}\frac{1}{2n_j}\sum\nolimits_{i=1}^{n_j}\bigg(\frac{k-j-1}{j}\sum\nolimits_{y=1}^k(2\alpha_{i,y}-1)\sigma_i\mathcal{L}\big(f(\boldsymbol{x}),y^\prime\big)+\sigma_i\mathcal{L}\big(f(\boldsymbol{x}),y^\prime\big)\bigg)\Bigg]\\
\nonumber
\leq& \mathbb{E}_{(\boldsymbol{x}_i,\widebar{Y}_i)\sim\widebar{p}(\boldsymbol{x},\widebar{Y}\mid s=j)}\mathbb{E}_{\boldsymbol{\sigma}}\Bigg[\sup_{f\in\mathcal{F}}\frac{1}{n_j}\sum\nolimits_{i=1}^{n_j}\sum\nolimits_{y=1}^k\sigma_i\mathcal{L}\big(f(\boldsymbol{x}_i),y\big)\Bigg]
\\
\nonumber
&\quad\quad\quad\quad\quad\quad\quad\quad\quad\quad\quad\quad +\mathbb{E}_{(\boldsymbol{x}_i,\widebar{Y}_i)\sim\widebar{p}(\boldsymbol{x},\widebar{Y}\mid s=j)}\mathbb{E}_{\boldsymbol{\sigma}}\Bigg[\sup_{f\in\mathcal{F}}\frac{1}{n_j}\sum\nolimits_{i=1}^{n_j}\frac{k-j-1}{j}\sum\nolimits_{y=1}^k\sigma_i\mathcal{L}\big(f(\boldsymbol{x}_i),y\big)\Bigg]\\
\nonumber
=&\frac{k-1}{j}\mathbb{E}_{(\boldsymbol{x}_i,\widebar{Y}_i)\sim\widebar{p}(\boldsymbol{x},\widebar{Y}\mid s=j)}\mathbb{E}_{\boldsymbol{\sigma}}\Bigg[\sup_{f\in\mathcal{F}}\frac{1}{n_j}\sum\nolimits_{i=1}^{n_j}\sum\nolimits_{y=1}^k\sigma_i\mathcal{L}\big(f(\boldsymbol{x}_i),y\big)\Bigg]\\
\nonumber
\leq& \frac{k-1}{j}\sum\nolimits_{y=1}^k\mathbb{E}_{\boldsymbol{x}_i\sim p(\boldsymbol{x})}\mathbb{E}_{\boldsymbol{\sigma}}\Bigg[\sup_{f\in\mathcal{F}}\frac{1}{n_j}\sum\nolimits_{i=1}^{n_j}\sigma_i\mathcal{L}\big(f(\boldsymbol{x}_i),y\big)\Bigg]\quad\quad (\because p(\boldsymbol{x}) = \widebar{p}(\boldsymbol{x}\mid s=j)),
\end{align}
%Here, we assume that the loss function $\mathcal{L}\big(f(\boldsymbol{x}),y\big)$ can be simplified as $\mathcal{L}\big(f(\boldsymbol{x}),y\big) = \phi(f_y(\boldsymbol{x}))$. As we have shown in Section 4.3, all the discussed loss functions in this paper satisfy this condition.
Then, we have
\begin{align}
\nonumber
\widebar{\mathfrak{R}}_{n_j}(\mathcal{H}_j)\leq& \frac{k-1}{j}\sum\nolimits_{y=1}^k\mathbb{E}_{\boldsymbol{x}_i\sim p(\boldsymbol{x})}\mathbb{E}_{\boldsymbol{\sigma}}\Bigg[\sup_{f\in\mathcal{F}}\frac{1}{n_j}\sum\nolimits_{i=1}^{n_j}\sigma_i\mathcal{L}\big(f(\boldsymbol{x}_i),y\big)\Bigg]\\
\nonumber
\leq& \frac{k-1}{j}\sum\nolimits_{y=1}^k\mathfrak{R}_{n_j}(\mathcal{L}\circ\mathcal{F})\\
\nonumber
\leq&\frac{\sqrt{2}\rho k(k-1)}{j}\sum\nolimits_{y=1}^k\mathfrak{R}_{n_j}(\mathcal{G}_y),
\end{align}
where we applied the Rademacher vector contraction inequality~\cite{maurer2016vector} in the last inequality.
\end{proof}
Under the assumptions described in the above three lemmas (Lemma~\ref{basic_lemma}, Lemma~\ref{temp_lemma}, and Lemma~\ref{rade_func}), for any $\delta>0$, with probability at least $1-\delta$,
\begin{gather}
\nonumber
R(\widehat{f}) - R(f^\star)
\leq \sum\nolimits_{j=1}^{k-1}p(s=j)\bigg(\frac{4\sqrt{2}\rho k(k-1)}{j}\sum\nolimits_{y=1}^k\mathfrak{R}_{n_j}(\mathcal{G}_y)+(4k-4j-2)C_{\mathcal{L}}\sqrt{\frac{\log\frac{2(k-1)}{\delta}}{2n_j}}\bigg).
\end{gather}
It is clear that by combining the above three lemmas, Theorem 4 is proved.
\qed
\section{Derivations and Boundness of the Used Loss Functions}\label{D}
\subsection{Derivations of the Used Loss Functions}
Conventionally, the label for each instance $\boldsymbol{x}$ is in one-hot encoding. Concretely, if the label of $\boldsymbol{x}$ is $y$, then we represent the label vector as $\boldsymbol{e}_y$ where $e_{yj}=1$ if $j=y$, otherwise 0. In this way, we provide the detailed derivations of CCE, MAE, and MSE as follows.
\begin{itemize}
\item Categorical Cross Entropy (CCE):
\begin{gather}
\nonumber
\mathcal{L}_{\text{CCE}}(f(\boldsymbol{x}),y)=-\sum\nolimits_{j=1}^ke_{yj}\log p_{\boldsymbol{\theta}}(j|\boldsymbol{x}) =
-\log p_{\boldsymbol{\theta}}(y|\boldsymbol{x}).
\end{gather}
\item Mean Absolute Error (MAE):
\begin{gather}
\nonumber
\mathcal{L}_{\text{MAE}}(f(\boldsymbol{x}),y)=\sum\nolimits_{j=1}^k\left|p_{\boldsymbol{\theta}}(j|\boldsymbol{x})-e_{yj}\right| = 2-2p_{\boldsymbol{\theta}}(y|\boldsymbol{x}).
\end{gather}
\item Mean Square Error (MSE):
\begin{gather}
\nonumber
\mathcal{L}_{\text{MSE}}(f(\boldsymbol{x}),y)=\sum\nolimits_{j=1}^k\big(p_{\boldsymbol{\theta}}(j|\boldsymbol{x})-e_{yj}\big)^2=1-2p_{\boldsymbol{\theta}}(y|\boldsymbol{x})+\sum\nolimits_{j=1}^kp_{\boldsymbol{\theta}}(j|\boldsymbol{x})^2.
\end{gather}
\end{itemize}
\subsection{Boundness of the Used Loss Functions}
\setcounter{table}{4}
\begin{table*}[!t]
\centering
	\caption{Statistics of the used benchmark datasets.}
	\label{datasets}
\setlength{\tabcolsep}{4.0mm}{
				\begin{tabular}{lccccc}
					\toprule
					Dataset & \#Train & \#Test & \#Features & \#Classes & Model \\
					\midrule
					MNIST & 60,000 & 10,000 & 784 & 10 & Linear Model, MLP ($d$-500-10)\\
					Fashion-MNIST & 60,000 & 10,000 & 784 & 10 & Linear Model, MLP ($d$-500-10) \\
					Kuzushiji-MNIST & 60,000 & 10,000 & 784 & 10 & Linear Model, MLP ($d$-500-10) \\
					20Newsgroups & 16,961 & 1,885 & 1,000 & 20 & Linear Model, MLP ($d$-500-20)\\
					CIFAR-10 & 50,000 & 10,000 & 3,072 & 10 & ResNet, DenseNet \\
					\midrule
					Yeast & 1,335 & 149 & 8 & 10 & Linear Model\\
					Texture & 4,950 & 550 & 40 & 11 & Linear Model\\%Phoneme & 4,863 & 541 & 5 & 2 & Linear Model\\
					Dermatology & 329 & 37 & 34 & 6 & Linear Model\\
					Synthetic Control & 540 & 60 & 60 & 6 & Linear Model \\
					\bottomrule
				\end{tabular}
				}
\end{table*}
\setcounter{figure}{1}
\begin{figure*}[!t]
\centering
\subfigure[MNIST, Linear]{
    %\label{fig:mini:subfig:a} %% label for first subfigure
      \includegraphics[width=1.6in]{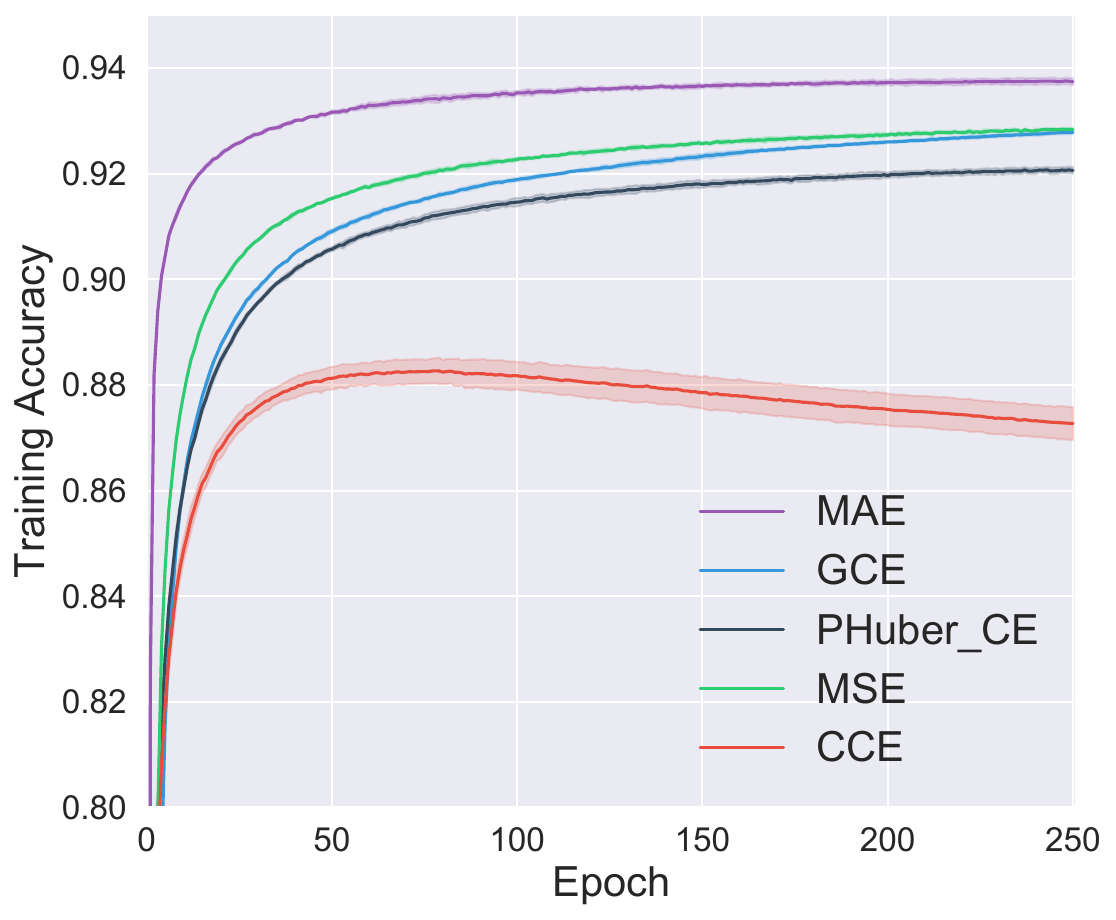}}%
  \subfigure[MNIST, MLP]{
    %\label{fig:mini:subfig:b} %% label for second subfigure
      \includegraphics[width=1.6in]{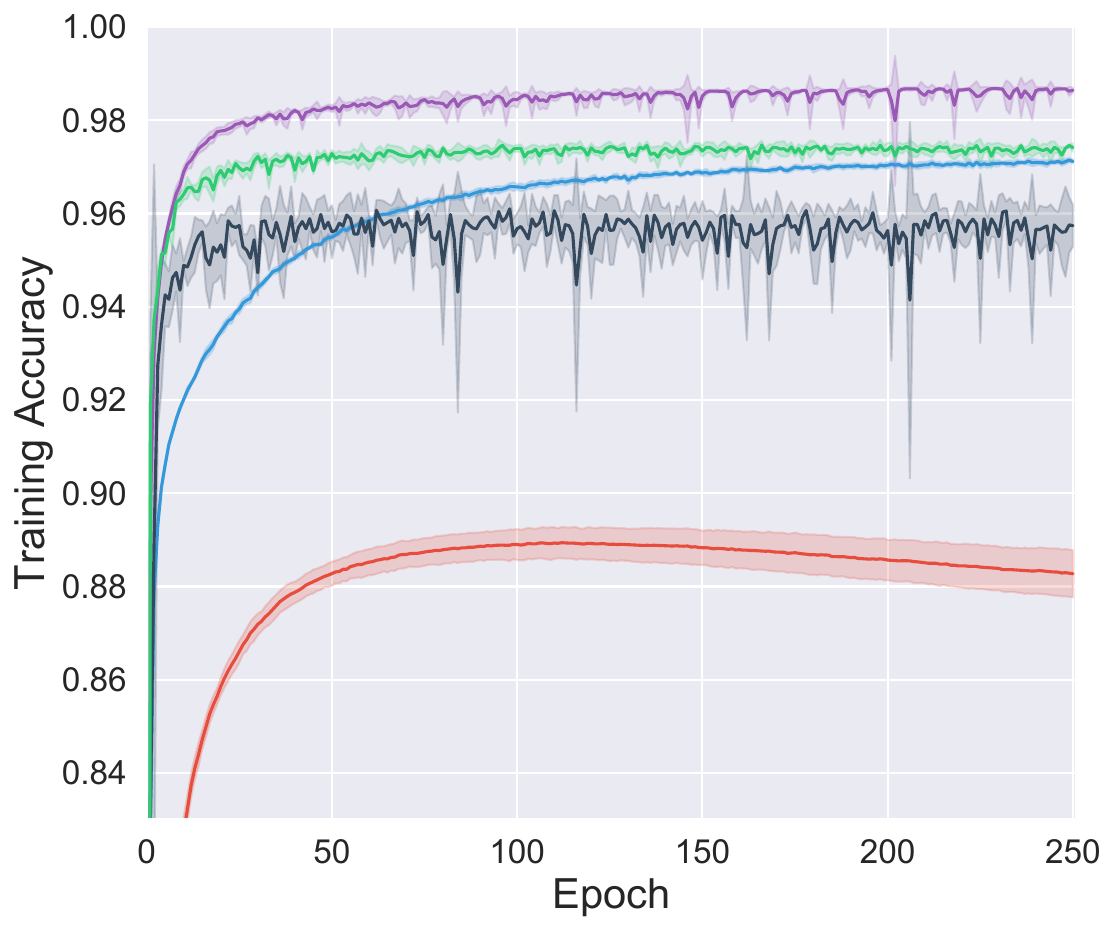}}
      \subfigure[Fashion-MNIST, Linear]{
    %\label{fig:mini:subfig:c} %% label for second subfigure
      \includegraphics[width=1.6in]{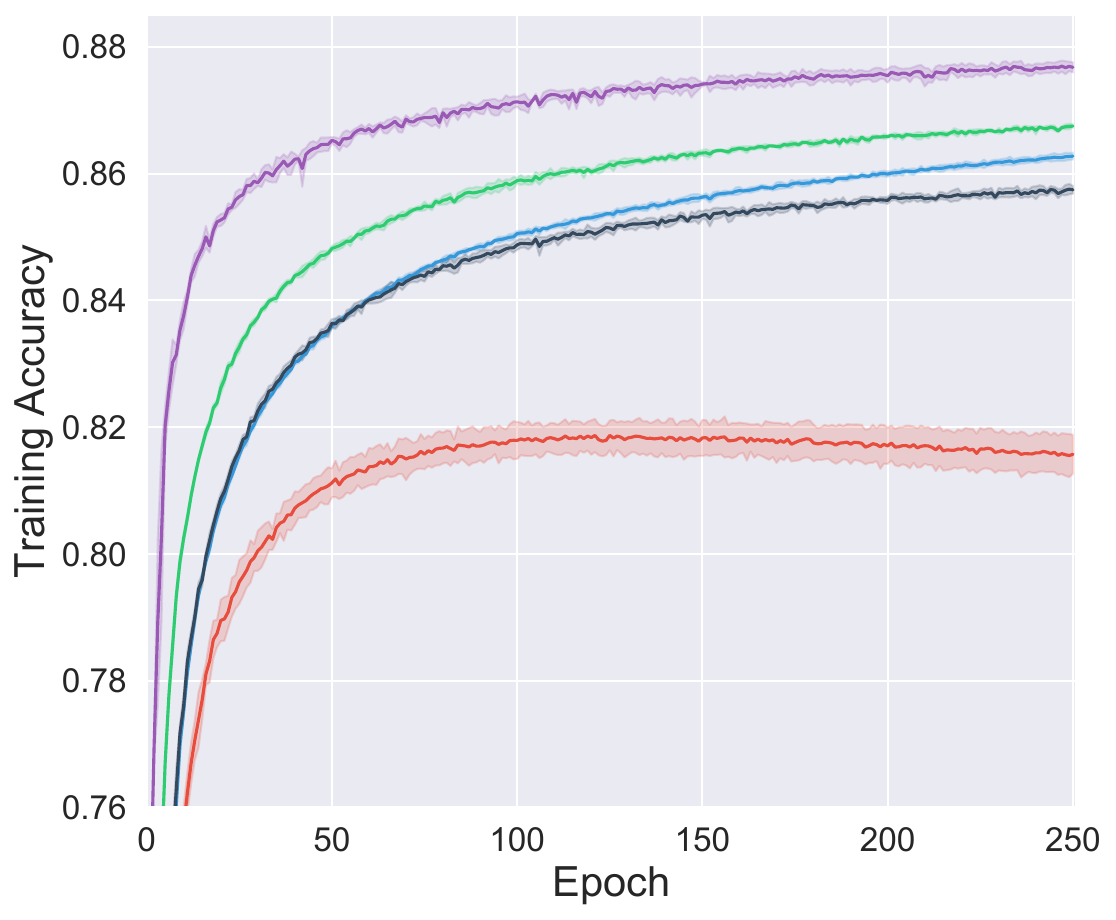}}
      \subfigure[Fashion-MNIST, MLP]{
    %\label{fig:mini:subfig:d} %% label for second subfigure
      \includegraphics[width=1.6in]{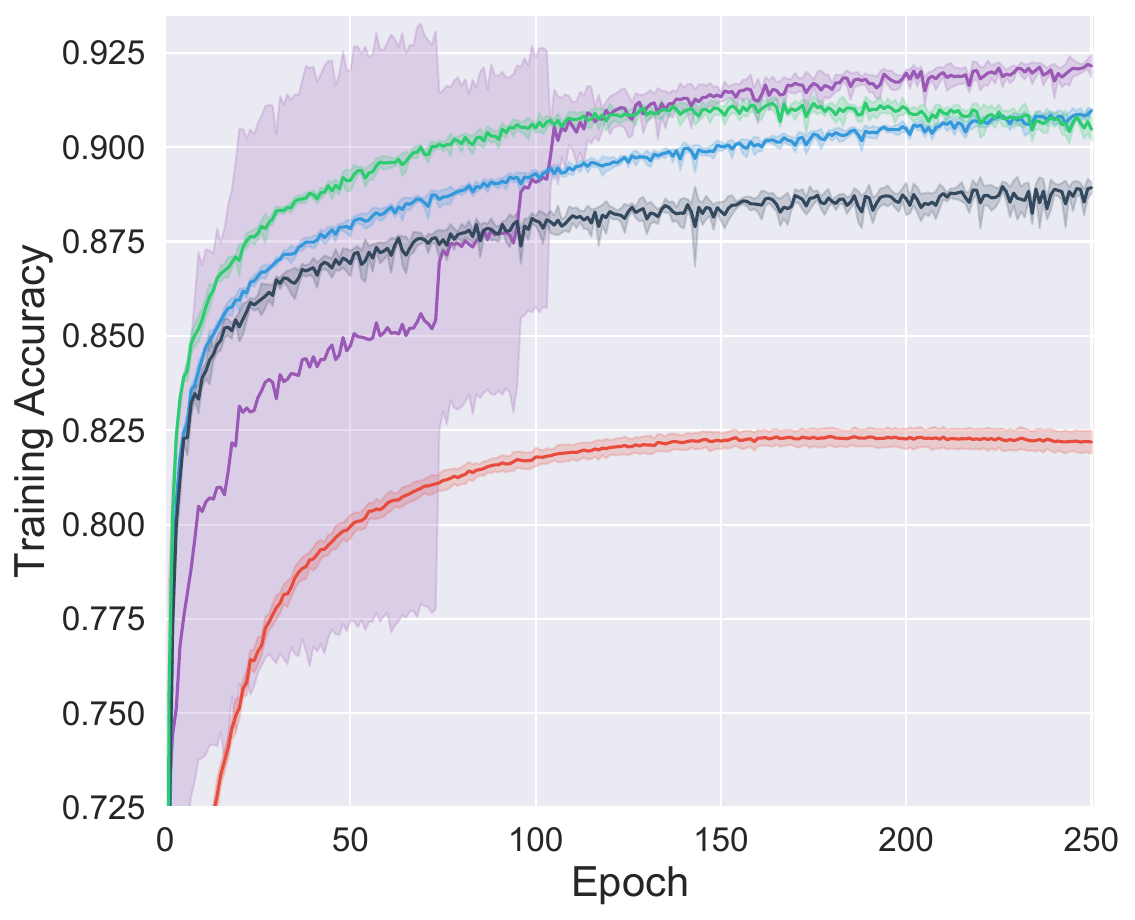}}
      \subfigure[Kuzushiji-MNIST, Linear]{
    %\label{fig:mini:subfig:d} %% label for second subfigure
      \includegraphics[width=1.6in]{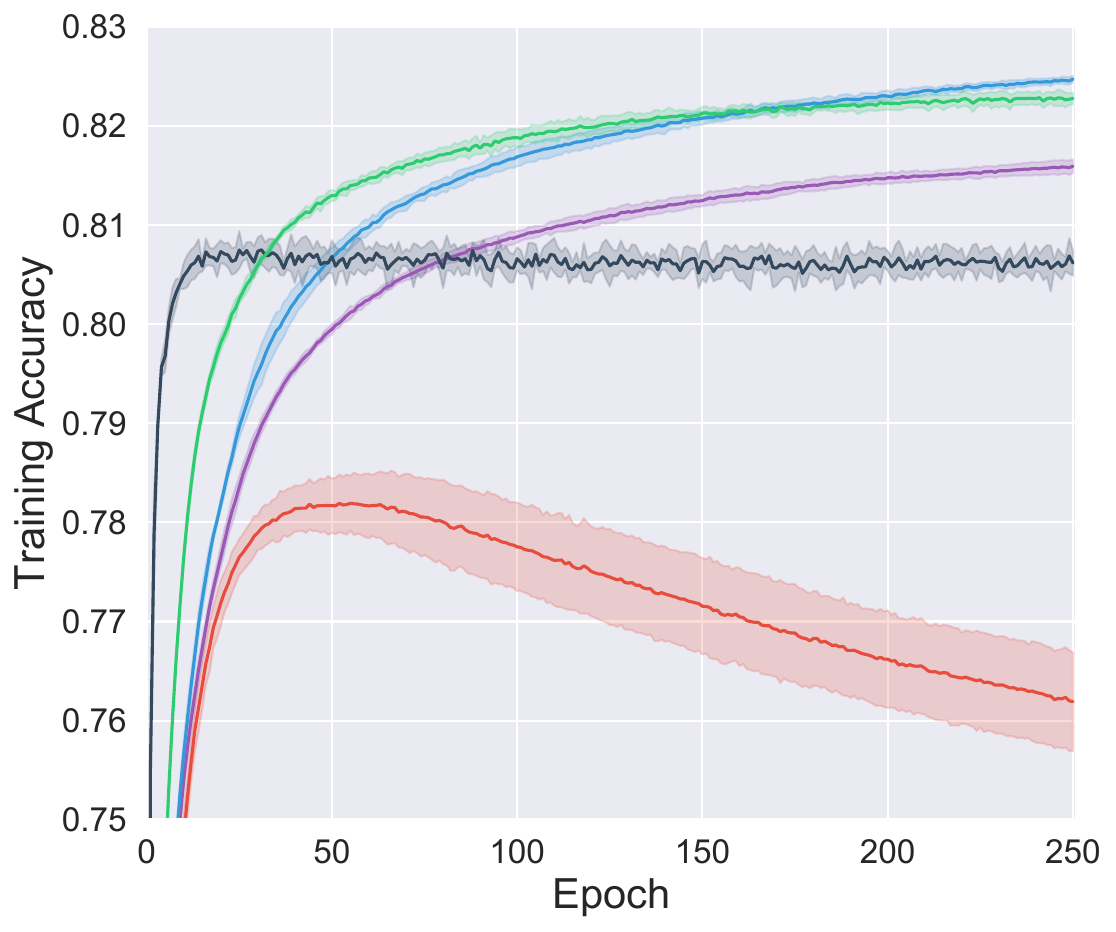}}
      \subfigure[Kuzushiji-MNIST, MLP]{
    %\label{fig:mini:subfig:d} %% label for second subfigure
      \includegraphics[width=1.6in]{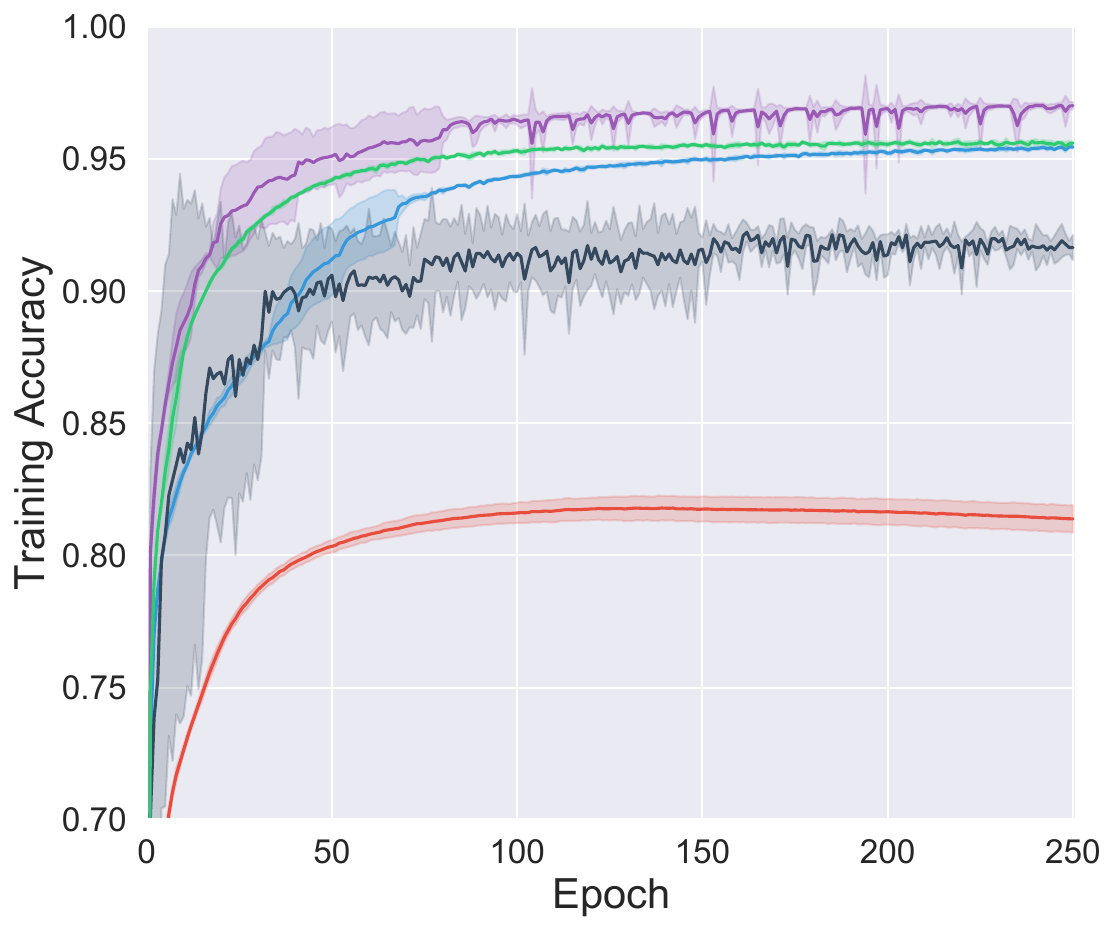}}
      \subfigure[CIFAR-10, Resnet]{
    %\label{fig:mini:subfig:d} %% label for second subfigure
      \includegraphics[width=1.6in]{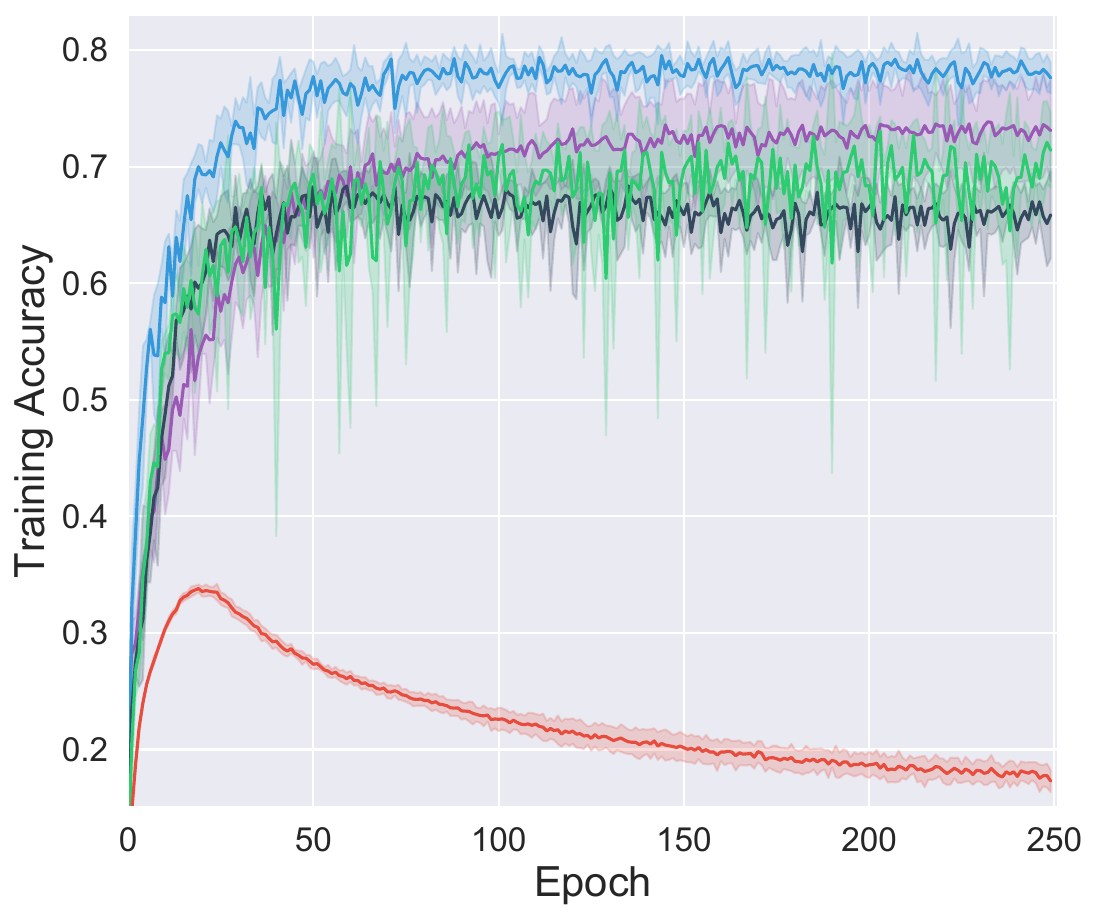}}
      \subfigure[CIFAR-10, Densenet]{
    %\label{fig:mini:subfig:d} %% label for second subfigure
      \includegraphics[width=1.6in]{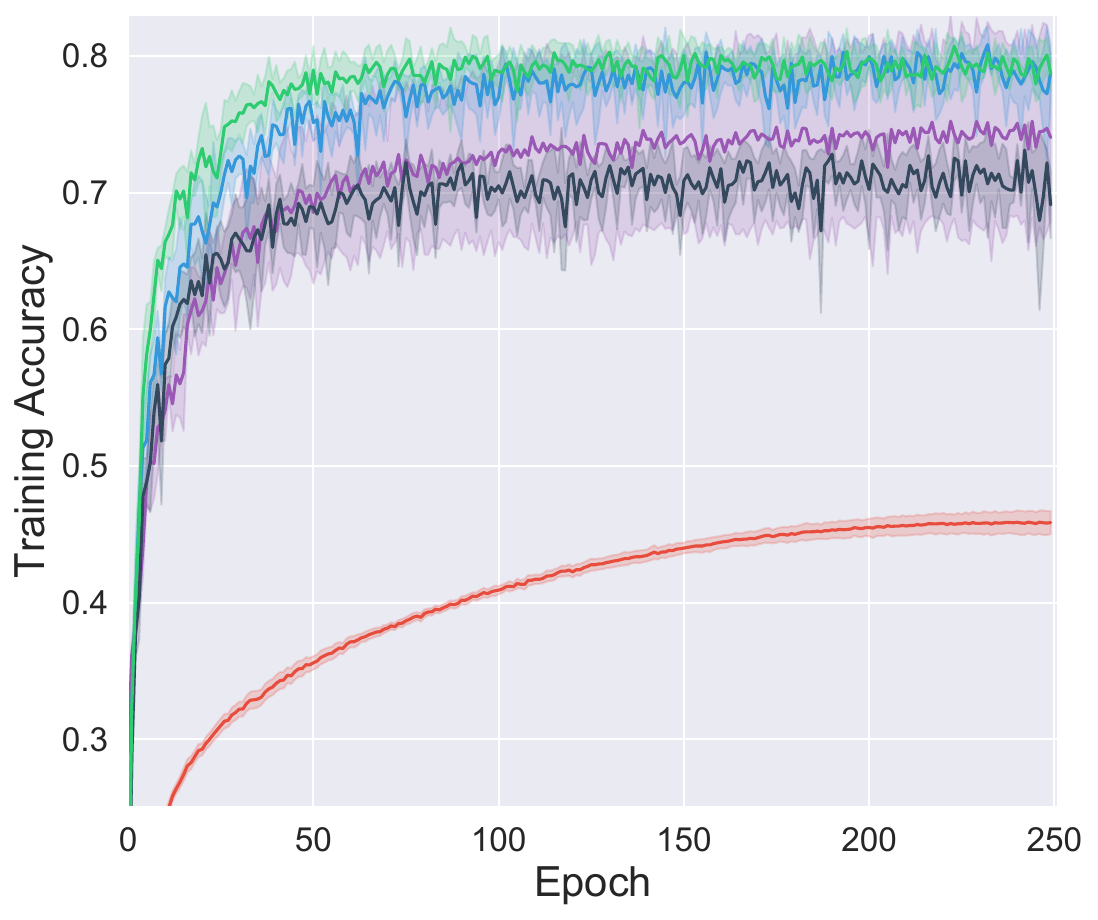}}
  \caption{Experimental results of different loss functions for different datasets and models. Dark colors show the mean accuracy of 5 trials and light colors show the standard deviation.}
  \label{training} %% label for entire figure
\end{figure*}
Firstly, it is clear that each loss function is non-negative. Besides, for each loss function, the loss becomes larger if $p_{\boldsymbol{\theta}}(y|\boldsymbol{x})$ gets smaller given the correct label $y$. Note that $0< p_{\boldsymbol{\theta}}(y|\boldsymbol{x})<1$, hence the upper bound of each loss function is stated as follows.
\begin{itemize}
\item MAE: $\mathcal{L}_{\text{MAE}}(f(\boldsymbol{x}),y) < 2$.
\item MSE: $\mathcal{L}_{\text{MSE}}(f(\boldsymbol{x}),y)< 1-0+\sum\nolimits_{j=1}^k p_{\boldsymbol{\theta}}(j|\boldsymbol{x})^2<2$.
\item GCE: $\mathcal{L}_{\text{GCE}}(f(\boldsymbol{x}),y)<1/q$ where $q=0.7$.
\item PHuber-CE: $\mathcal{L}_{\text{PHuber-CE}}(f(\boldsymbol{x}),y)<\log\tau+1$ where $\tau=10$.
\end{itemize}
Note that for CCE, $\mathcal{L}_{\text{CCE}}(f(\boldsymbol{x}),y) < -\log 0 = \infty$. Therefore, we can know that MAE, MSE, GCE, and PHuber-CE are upper-bounded, while CCE is not upper-bounded.
\section{Additional Information of Experiments}
\begin{table*}[!t]
\centering
\caption{Classification accuracy (\%) of each approach on Kuzushiji-MNIST using linear model. The best performance is highlighted in boldface.}
\label{fixed_linear}
\resizebox{1.0\textwidth}{!}{
\setlength{\tabcolsep}{3.0mm}{
\begin{tabular}{c|c|c|c|c|c|c|c|c|c}
\toprule
\multicolumn{2}{c|}{Approach}              & $s=1$ & $s=2$ & $s=3$ & $s=4$ & $s=5$ & $s=6$ & $s=7$ & $s=8$\\
\midrule
\multirow{4}{*}{Upper-bound Losses }
                              & EXP       &  \bf{60.87}  &  \bf{62.73}   &  \bf{63.53}  & \bf{64.03} & \bf{64.55} & 65.06 & 65.23 & 65.65
                         \\
                         & 	&  ($\pm$0.38)    &  ($\pm$0.58)    &  ($\pm$0.30)    &  ($\pm$0.38)    &  ($\pm$0.41)    &  ($\pm$0.15)    &  ($\pm$0.10)    &  ($\pm$0.08)    \\

                              & LOG &  60.11  & 61.57 & 62.71 & 63.36 & 64.01 & 65.68 & \bf{69.35} & \bf{70.10} \\
                              & 	&  ($\pm$0.49)    &  ($\pm$0.15)    &  ($\pm$0.32)    &  ($\pm$0.09)    &  ($\pm$0.13)    &  ($\pm$0.27)    &  ($\pm$0.22)    &  ($\pm$0.18)    \\
\midrule
\multirow{8}{*}{Bounded Losses}
& MAE     &  60.43  & 62.71 & 63.51  & 63.75 & 63.94 & 64.61 & 64.82 & 65.10 \\
& 	&  ($\pm$0.43)    &  ($\pm$0.45)    &  ($\pm$0.10)    &  ($\pm$0.31)    &  ($\pm$0.38)    &  ($\pm$0.19)    &  ($\pm$0.16)    &  ($\pm$0.16)    \\

& MSE  &	58.97    &    62.07    &    63.05    &    63.85    &    64.47    &    64.80    &    65.17    &    65.43    \\
& 	&  ($\pm$0.47)    &  ($\pm$0.54)    &  ($\pm$0.38)    &  ($\pm$0.57)    &  ($\pm$0.43)    &  ($\pm$0.34)    &  ($\pm$0.25)    &  ($\pm$0.10)    \\

& GCE       & 60.48  & 62.71 & 63.13 & 63.87 & 63.91 & 64.28 & 64.38 & 64.33\% \\
& 	&  ($\pm$0.55)    &  ($\pm$0.65)    &  ($\pm$0.30)    &  ($\pm$0.33)    &  ($\pm$0.30)    &  ($\pm$0.07)    &  ($\pm$0.12)    &  ($\pm$0.06)    \\

& Phuber-CE & 52.69  & 56.58   & 61.10   & 62.32    & 64.51   & 64.93  & 65.96   &  65.81  \\
&  &  ($\pm$4.22)    &  ($\pm$3.94)    &  ($\pm$2.58)    &  ($\pm$1.50)    &  ($\pm$0.68)    &  ($\pm$0.52)    &  ($\pm$0.37)    &  ($\pm$0.62)    \\
\midrule
\multirow{2}{*}{Unbounded Loss}                & CCE  &	51.59    &    55.98    &    59.15    &    61.08    &    63.19    &    65.05    &    66.82    &    68.23    \\
& 	&  ($\pm$0.64)    &  ($\pm$1.26)    &  ($\pm$1.18)    &  ($\pm$0.78)    &  ($\pm$0.54)    &  ($\pm$0.51)    &  ($\pm$0.41)    &  ($\pm$0.21)    \\
\midrule
\multirow{10}{*}{Decomposition before Shuffle}  & GA  &	51.72    &    53.78    &    54.58    &    54.78    &    55.33    &    55.67    &    55.91    &    56.15    \\
& 	&  ($\pm$1.04)    &  ($\pm$1.07)    &  ($\pm$0.87)    &  ($\pm$0.58)    &  ($\pm$0.29)    &  ($\pm$0.31)    &  ($\pm$0.42)    &  ($\pm$0.23)    \\

& NN  &	55.03    &    57.68    &    58.87    &    59.52    &    60.41    &    60.89    &    61.41    &    61.62    \\
& 	&  ($\pm$1.35)    &  ($\pm$1.29)    &  ($\pm$1.19)    &  ($\pm$0.87)    &  ($\pm$0.59)    &  ($\pm$0.53)    &  ($\pm$0.36)    &  ($\pm$0.09)    \\

& FREE &	57.26    &    60.69    &    62.77    &    63.91    &    64.54    &    \textbf{66.21}    &    67.00    &    67.71    \\
& 	&  ($\pm$0.83)    &  ($\pm$0.96)    &  ($\pm$0.79)    &  ($\pm$0.65)    &  ($\pm$0.55)    &  ($\pm$0.56)    &  ($\pm$0.28)    &  ($\pm$0.20)    \\

& PC &	54.31    &    58.11    &    60.15    &    61.32    &    62.56    &    63.55    &    64.27    &    65.16    \\
& 	&  ($\pm$1.04)    &  ($\pm$0.87)    &  ($\pm$0.79)    &  ($\pm$0.68)    &  ($\pm$0.59)    &  ($\pm$0.43)    &  ($\pm$0.20)    &  ($\pm$0.18)    \\

& Forward &	60.05    &    61.53    &    62.43    &    62.98    &    63.48    &    63.95    &    64.14    &    64.27    \\
& 	&  ($\pm$0.43)    &  ($\pm$0.31)    &  ($\pm$0.26)    &  ($\pm$0.40)    &  ($\pm$0.34)    &  ($\pm$0.29)    &  ($\pm$0.09)    &  ($\pm$0.16)    \\
\midrule
\multirow{10}{*}{Decomposition after Shuffle}  & GA   &	51.72    &    53.79    &    54.59    &    54.83    &    55.33    &    55.67    &    55.90    &    56.18    \\
& 	&  ($\pm$1.05)    &  ($\pm$1.07)    &  ($\pm$0.85)    &  ($\pm$0.58)    &  ($\pm$0.35)    &  ($\pm$0.31)    &  ($\pm$0.41)    &  ($\pm$0.22)    \\

& NN &	55.03    &    58.58    &    60.43    &    61.58    &    62.99    &    64.00    &    65.07    &    66.08    \\
& 	&  ($\pm$1.35)    &  ($\pm$1.11)    &  ($\pm$1.00)    &  ($\pm$0.72)    &  ($\pm$0.49)    &  ($\pm$0.48)    &  ($\pm$0.36)    &  ($\pm$0.10)    \\

& FREE  &	57.26    &    60.32    &    62.11    &    62.98    &    64.30    &    65.18    &    66.02    &    67.02    \\
& 	&  ($\pm$0.84)    &  ($\pm$0.94)    &  ($\pm$0.64)    &  ($\pm$0.67)    &  ($\pm$0.47)    &  ($\pm$0.45)    &  ($\pm$0.28)    &  ($\pm$0.18)    \\

& PC &	54.31    &    57.32    &    58.95    &    60.17    &    61.47    &    62.54    &    63.53    &    64.74    \\
& 	&  ($\pm$1.04)    &  ($\pm$0.76)    &  ($\pm$0.77)    &  ($\pm$0.83)    &  ($\pm$0.45)    &  ($\pm$0.40)    &  ($\pm$0.22)    &  ($\pm$0.22)    \\

& Forward  &	60.02    &    61.75    &    62.68    &    63.19    &    63.59    &    63.94    &    64.18    &    64.32    \\
& 	&  ($\pm$0.44)    &  ($\pm$0.25)    &  ($\pm$0.23)    &  ($\pm$0.28)    &  ($\pm$0.19)    &  ($\pm$0.09)    &  ($\pm$0.14)    &  ($\pm$0.15)    \\
\bottomrule
\end{tabular}
}
}
\end{table*}
\begin{table*}[!t]
\centering
\caption{Classification accuracy (\%) of each approach on Kuzushiji-MNIST using MLP. The best performance is highlighted in boldface.}
\label{fixed_mlp}
\resizebox{1.0\textwidth}{!}{
\setlength{\tabcolsep}{3.0mm}{
\begin{tabular}{c|c|c|c|c|c|c|c|c|c}
\toprule
\multicolumn{2}{c|}{Approach}              & $s=1$ & $s=2$ & $s=3$ & $s=4$ & $s=5$ & $s=6$ & $s=7$ & $s=8$\\
\midrule
\multirow{4}{*}{Upper-bound Losses }
& EXP & 71.66 & \bf{82.51} & 84.45 & 87.10 & 88.35 & \bf{89.61} & 90.18 & 90.92 \\
& 	&    ($\pm$3.48 )    &    ($\pm$3.08 )    &    ($\pm$0.24 )    &    ($\pm$0.37 )    &    ($\pm$0.18 )    &    ($\pm$0.33 )    &    ($\pm$0.37 )    &    ($\pm$0.15)    \\
& LOG & \bf{77.07} & 82.39 & \bf{85.54} & \bf{87.60} & \bf{88.87} & 89.25 & \bf{90.22} & \bf{91.19}\\
& 	&    ($\pm$3.00 )    &    ($\pm$0.73 )    &    ($\pm$0.35 )    &    ($\pm$0.40 )    &    ($\pm$0.34 )    &    ($\pm$0.37 )    &    ($\pm$0.31 )    &    ($\pm$0.11)    \\
\midrule
\multirow{8}{*}{Bounded Losses}
& MAE &	69.87    &    73.60    &    79.97    &    85.34    &    86.91    &    89.10    &    90.32    &    91.06    \\
& 	&    ($\pm$1.04 )    &    ($\pm$5.77 )    &    ($\pm$3.71 )    &    ($\pm$2.78 )    &    ($\pm$3.06 )    &    ($\pm$0.46 )    &    ($\pm$0.31 )    &    ($\pm$0.34)    \\

& MSE  &	57.56    &    71.37    &    78.26    &    82.97    &    85.37    &    86.82    &    88.03    &    88.69    \\
& 	&    ($\pm$0.92 )    &    ($\pm$0.89 )    &    ($\pm$0.49 )    &    ($\pm$0.41 )    &    ($\pm$0.45 )    &    ($\pm$0.13 )    &    ($\pm$0.11 )    &    ($\pm$0.05)    \\

& GCE &	63.85    &    74.11    &    79.18    &    83.65    &    85.23    &    86.32    &    87.12    &    87.64    \\
& 	&    ($\pm$1.27 )    &    ($\pm$2.38 )    &    ($\pm$2.31 )    &    ($\pm$0.15 )    &    ($\pm$0.25 )    &    ($\pm$0.27 )    &    ($\pm$0.20 )    &    ($\pm$0.09)    \\

& Phuber-CE & 10.24 & 14.76 & 26.60 & 73.43 & 81.41 & 83.00 & 84.69 & 85.59 \\
&  &    ($\pm$4.09 )    &    ($\pm$2.11 )    &    ($\pm$1.58 )    &    ($\pm$1.50 )    &    ($\pm$0.58 )    &    ($\pm$0.42 )    &    ($\pm$0.47 )    &    ($\pm$0.52)    \\
\midrule
\multirow{2}{*}{Unbounded Loss}                & CCE  &	56.17    &    60.89    &    64.18    &    66.57    &    69.14    &    71.63    &    74.55    &    78.22    \\
& 	&    ($\pm$0.64 )    &    ($\pm$0.61 )    &    ($\pm$0.77 )    &    ($\pm$0.41 )    &    ($\pm$0.49 )    &    ($\pm$0.31 )    &    ($\pm$0.31 )    &    ($\pm$0.22)    \\
\midrule
\multirow{10}{*}{Decomposition before Shuffle}  & GA  &	70.25    &    76.50    &    79.77    &    82.03    &    84.05    &    85.58    &    86.40    &    87.49    \\
& 	&    ($\pm$0.24 )    &    ($\pm$0.47 )    &    ($\pm$0.32 )    &    ($\pm$0.22 )    &    ($\pm$0.64 )    &    ($\pm$0.32 )    &    ($\pm$0.24 )    &    ($\pm$0.15)    \\

& NN  &	65.33    &    71.34    &    75.46    &    78.67    &    81.40    &    84.08    &    86.56    &    88.61    \\
& 	&    ($\pm$0.51 )    &    ($\pm$0.53 )    &    ($\pm$0.31 )    &    ($\pm$0.58 )    &    ($\pm$0.28 )    &    ($\pm$0.16 )    &    ($\pm$0.39 )    &    ($\pm$0.12)    \\

& FREE &	53.90    &    60.32    &    63.98    &    66.79    &    69.31    &    71.65    &    74.43    &    76.61    \\
& 	&    ($\pm$1.05 )    &    ($\pm$1.14 )    &    ($\pm$0.85 )    &    ($\pm$0.64 )    &    ($\pm$0.73 )    &    ($\pm$0.73 )    &    ($\pm$0.28 )    &    ($\pm$0.33)    \\

& PC &	56.36    &    62.37    &    66.09    &    69.51    &    72.46    &    75.18    &    78.50    &    82.40    \\
& 	&    ($\pm$0.56 )    &    ($\pm$0.50 )    &    ($\pm$0.44 )    &    ($\pm$0.47 )    &    ($\pm$0.35 )    &    ($\pm$0.33 )    &    ($\pm$0.52 )    &    ($\pm$0.38)    \\

& Forward &	75.40    &    83.19    &    85.18    &    86.63    &    87.51    &    88.29    &    88.96    &    89.41    \\
& 	&    ($\pm$2.02 )    &    ($\pm$0.61 )    &    ($\pm$0.48 )    &    ($\pm$0.38 )    &    ($\pm$0.29 )    &    ($\pm$0.29 )    &    ($\pm$0.26 )    &    ($\pm$0.25)    \\
\midrule
\multirow{10}{*}{Decomposition after Shuffle}  & GA   &	70.25    &    75.91    &    78.46    &    80.60    &    82.14    &    83.48    &    84.01    &    84.65    \\
& 	&    ($\pm$0.24 )    &    ($\pm$1.37 )    &    ($\pm$2.84 )    &    ($\pm$3.35 )    &    ($\pm$4.51 )    &    ($\pm$4.92 )    &    ($\pm$5.35 )    &    ($\pm$6.28)    \\

& NN &	63.73    &    67.26    &    69.46    &    71.25    &    73.15    &    74.82    &    77.09    &    79.39    \\
& 	&    ($\pm$0.97 )    &    ($\pm$0.82 )    &    ($\pm$0.74 )    &    ($\pm$0.62 )    &    ($\pm$0.45 )    &    ($\pm$0.35 )    &    ($\pm$0.17 )    &    ($\pm$0.21)    \\

& FREE &	55.33    &    60.81    &    64.65    &    67.01    &    69.60    &    71.63    &    74.22    &    77.16    \\
& 	&    ($\pm$0.89 )    &    ($\pm$0.97 )    &    ($\pm$0.89 )    &    ($\pm$0.70 )    &    ($\pm$0.78 )    &    ($\pm$0.46 )    &    ($\pm$0.40 )    &    ($\pm$0.50)    \\

& PC &	56.68    &    61.07    &    63.86    &    65.61    &    68.03    &    69.74    &    72.49    &    75.17    \\
& 	&    ($\pm$1.28 )    &    ($\pm$0.99 )    &    ($\pm$0.67 )    &    ($\pm$0.44 )    &    ($\pm$0.64 )    &    ($\pm$0.65 )    &    ($\pm$0.37 )    &    ($\pm$0.46)    \\

& Forward  &	66.09    &    73.20    &    75.76    &    82.53    &    86.27    &    88.05    &    89.24    &    90.22    \\
& 	&    ($\pm$0.49 )    &    ($\pm$3.05 )    &    ($\pm$2.61 )    &    ($\pm$2.60 )    &    ($\pm$0.65 )    &    ($\pm$0.27 )    &    ($\pm$0.22 )    &    ($\pm$0.20)    \\
\bottomrule
\end{tabular}
}
}
\end{table*}
\subsection{Datasets and Models}\label{E.1}
In the experiments of Section 5, we use 5 widely-used large-scale benchmark datasets and 4 regular-scale datasets from the UCI Machine Learning Repository. The statistics of these datasets with the corresponding base models are reported in Table~\ref{datasets}. Hyper-parameters for all the approaches are selected so as to maximize the accuracy on a validation set, which is constructed by randomly sampling 10\% of the training set.
We report the characteristics, the parameter settings (to reproduce the experimental results), and the sources of these datasets as follows.
\begin{itemize}
\item MNIST~\cite{lecun1998gradient}: It is a 10-class dataset of handwritten digits (0 to 9). Each instance is a 28$\times$28 grayscale image. %Parameter settings for EXP: learning rate $10^{-4}$ and weight decay $10^{-6}$ for linear model; learning rate $10^{-3}$ and weight decay $10^{-5}$ for MLP. Parameter settings for LOG: learning rate $10^{-4}$ and weight decay $10^{-6}$ for linear model; learning rate $10^{-3}$ and weight decay $10^{-5}$ for MLP.
Source:
\url{http://yann.lecun.com/exdb/mnist/}
\item Kuzushiji-MNIST~\cite{clanuwat2018deep}:
It is a 10-class dataset of cursive Japanese (``Kuzushiji") characters. Each instance is a 28$\times$28 grayscale image. %Parameter settings for EXP: learning rate $10^{-3}$ and weight decay $10^{-5}$ for linear model; learning rate $10^{-3}$ and weight decay $10^{-5}$ for MLP. Parameter settings for LOG: learning rate $10^{-4}$ and weight decay $10^{-6}$ for linear model; learning rate $10^{-3}$ and weight decay $10^{-6}$ for MLP. Source: \url{https://github.com/zalandoresearch/fashion-mnist}
\item Fashion-MNIST~\cite{xiao2017fashion}:
It is a 10-class dataset of fashion items (T-shirt/top, trouser, pullover, dress, sandal, coat, shirt, sneaker, bag, and ankle boot). Each instance is a 28$\times$28 grayscale image. %Parameter settings for EXP: learning rate $10^{-4}$ and weight decay $10^{-4}$ for linear model; learning rate $10^{-3}$ and weight decay $10^{-5}$ for MLP. Parameter settings for LOG: learning rate $10^{-4}$ and weight decay $10^{-5}$ for linear model; learning rate $10^{-3}$ and weight decay $10^{-4}$ for MLP.
Source: \url{https://github.com/rois-codh/kmnist}
\item CIFAR-10~\cite{krizhevsky2009learning}: It is a 10-class dataset of 10 different objects (airplane, bird, automobile, cat, deer, dog, frog, horse, ship, and truck). Each instance is a 32$\times$32$\times$3 colored image in RGB format. This dataset is normalized with mean $(0.4914, 0.4822, 0.4465)$ and standard deviation $(0.247, 0.243, 0.261)$. %Parameter settings for EXP: learning rate $10^{-3}$ and weight decay $10^{-3}$ for ResNet; learning rate $10^{-3}$ and weight decay $10^{-5}$ for DenseNet. Parameter settings for LOG: learning rate $10^{-2}$ and weight decay $10^{-6}$ for linear model; learning rate $10^{-2}$ and weight decay $10^{-5}$ for MLP.
Source: \url{https://www.cs.toronto.edu/~kriz/cifar.html}
\item 20Newsgroups: It is a 20-class dataset of 20 different newsgroups (comp.graphics, comp.os.ms-windows.misc, comp.sys.ibm.pc.hardware, comp.sys.mac.hardware, comp.windows.x,
rec.autos, rec.motorcycles, rec.sport.baseball, rec.sport.hockey, sci.crypt, sci.electronics, sci.med, sci.space, misc.forsale, talk.politics.misc, talk.politics.guns, talk.politics.mideast, talk.religion.misc, alt.atheism, soc.religion.christian).
We obtained the tf-idf features, and applied
TruncatedSVD~\cite{halko2011finding} to reduce the dimension to 1000. We randomly sample 90\% of the examples from the whole dataset to construct the training set, and the rest 10\% forms the test set.
Source: \url{http://qwone.com/~jason/20Newsgroups/}
\item Yeast, Texture, Dermatology, Synthetic Control: They are all the datasets from the UCI Machine Learning Repository. Since they are all regular-scale datasets, we only apply linear model on them. For each dataset, we randomly sample 90\% of the examples from the whole dataset to construct the training set, and the rest 10\% forms the test set.
%We do not report all the parameter settings here, in order to save space.
The detailed parameter settings can be found in our provided code package. Source:
\url{https://archive.ics.uci.edu/ml/datasets.php}
\end{itemize}
For the used models, the detailed information of the used 34-layer ResNet~\cite{he2016deep} and 22-layer DenseNet~\cite{huang2017densely} can be found in the corresponding papers.

\subsection{Experimental Results on Training Accuracy}\label{E.2}
Here, we report the mean and standard deviation of training
accuracy (the training set is evaluated with ordinary labels)
of 5 trials in Figure~\ref{training}, to compare the bounded loss functions MAE, MSE, GCE, PHuber-CE, and the unbounded loss function CCE. The training accuracy can reflect the ability of the loss function in identifying the correct label from the non-complementary labels.

From Figure~\ref{training}, we can find that CCE always achieves the worst performance among all the loss functions, which implies that unbounded loss function is worse than bounded loss function, using our provided empirical risk estimator. This observation clearly supports our conjecture that the negative term in our empirical risk estimator could cause the over-fitting issue. In addition, we can also find that compared with other bounded loss functions, MAE achieves comparable performance in most cases, while it is sometimes inferior to other bounded losses due to its optimization
issue~\cite{zhang2018generalized}. All the above observations on the training accuracy (Figure~\ref{training}) are very similar to those observations on the test accuracy (Figure \ref{fig1} in our paper).

\subsection{Experimental Results on Fixed Complementary Label Set}\label{E.4}
We also conduct additional experiments to investigate the influence of the variable $s$ on Kuzushiji-MNIST using both linear model and MLP. Specifically, we study the case where
the size of each complementary label set $s$ is fixed at $j$
(i.e., $p(s=j)=1$) while increasing $j$ from 1 to $k-2$. The detailed experimental results are shown in Table~\ref{fixed_linear} and Table~\ref{fixed_mlp}. From the two tables, we can find that the (test) classification accuracy of
our approaches increases as $j$ increases. This observation
is clearly in accordance with our derived estimation error
bound (Theorem 4), as the estimation error would decrease
if $j$ increases. In addition, as shown in the two tables, our proposed upper-bound losses outperform other approaches in most cases. This observation also demonstrates the effectiveness of our proposed upper-bound losses.

\end{document}